\documentclass[11pt,a4paper]{article}

\usepackage[margin=1in]{geometry}
\usepackage{amsmath,amssymb,amsthm}
\usepackage{mathtools}
\usepackage[hidelinks]{hyperref}
\usepackage{enumitem}

\theoremstyle{plain}
\newtheorem{theorem}{Theorem}[section]
\newtheorem{proposition}[theorem]{Proposition}
\newtheorem{lemma}[theorem]{Lemma}
\newtheorem{corollary}[theorem]{Corollary}
\theoremstyle{definition}
\newtheorem{definition}[theorem]{Definition}
\newtheorem{example}[theorem]{Example}
\newtheorem{problem}[theorem]{Open problem}
\theoremstyle{remark}
\newtheorem{remark}[theorem]{Remark}

\DeclareMathOperator{\Lip}{Lip}
\DeclareMathOperator{\supp}{supp}
\DeclareMathOperator{\dist}{dist}

\DeclareMathOperator{\diag}{diag}
\newcommand{\R}{\mathbb{R}}
\newcommand{\one}{\mathbf{1}}
\newcommand{\Pcal}{\mathcal{P}}
\newcommand{\Scal}{\mathcal{S}}
\newcommand{\Bo}{\mathcal{B}_{\mathrm{orb}}}
\newcommand{\opi}[1]{\|#1\|_{\infty\to\infty}}

\title{\textbf{Stability-Constrained Approximation in Spline KANs:}\\
Exact Layer Balancing and Budget-Compatible Saturation}
\author{Aleksander Tankman\thanks{Fivestar Europe O\"U, Tallinn, Estonia. \texttt{aleksander.tankman@gmail.com}}}
\date{\today}

\begin{document}
\maketitle

\begin{abstract}
Deep spline superposition networks face a tension between approximation order and stability across depth. We study approximation under a hard layerwise Lipschitz budget, and organise it around two quantities: the \emph{factorisation stability complexity} of a given deep factorisation, and the \emph{budget-compatible approximation complexity} of a discretisation operator.

First, we solve exactly the finite-depth diagonal balancing problem for a fixed chain of nonnegative envelope matrices: the optimal uniform layer budget equals
\[
\bigl\|M_{L-1}\cdots M_0\bigr\|_{\infty\to\infty}^{1/L},
\]
attained by an explicit one-pass minimiser, for rectangular layers, with a complete treatment of degeneracies and non-attainment. The optimum can be arbitrarily larger than the Lipschitz constant of the network itself, because passing to envelopes destroys sign cancellation.

Second, we give a constructive spline discretisation theorem preserving the budget up to a controlled slack, with an explicit grid threshold. Conversely, for linear spline-valued operators that preserve the budget \emph{exactly}, we prove budget-compatible minimax lower bounds on classes constrained simultaneously in the first and third derivative norms --- a constraint pair that is forced by the problem and that rules out the usual scaling escapes.

Finally, we show that the corresponding layer errors need not cancel under composition: for every operator of the class there is a stable depth-$L$ tower realising a constant fraction of the accumulated error, so the linear-in-depth accumulation of the upper bound is not a proof artefact.
\end{abstract}



\tableofcontents

\bigskip

\section{Introduction}
\label{sec:intro}

\subsection{A small Lipschitz constant does not mean a stable factorisation}

Consider a deep composition $F=\Phi_{L-1}\circ\cdots\circ\Phi_0$ and the two-layer example with linear edges
\[
W_0=\begin{pmatrix}1&K\\0&1\end{pmatrix},\qquad
W_1=\begin{pmatrix}1&-K\\0&1\end{pmatrix},\qquad W_1W_0=I .
\]
Here $\Lip_\infty(F)=1$: the network is as stable as a function can be. Yet the two layers stretch and unstretch by a factor $K$, and no rescaling of the hidden coordinates can hide this. Corollary~\ref{cor:gap} makes it exact: the best uniform layer constant achievable by any diagonal recalibration is $\sqrt{1+2K}$, which tends to infinity. The reason is that the object controlling the layers is the matrix of edge Lipschitz constants, and passing to it destroys signs: $|W_1|\,|W_0|\ne|W_1W_0|$.

This example is the whole paper in miniature. It says that a budget on the \emph{layers} is a genuinely different constraint from a bound on the \emph{function}, and it raises two questions.

\medskip
\noindent\textbf{Q1 (allocation).} Given a factorisation, how evenly can the intermediate stretching be distributed among the layers?

\noindent\textbf{Q2 (approximation).} If a discretisation must preserve that layer budget, how much accuracy is it forced to lose?
\medskip

We answer Q1 exactly, and we answer Q2 from both sides: constructively from above, and by a minimax lower bound from below.

\subsection{Two complexity quantities}

For a factorisation $\boldsymbol\Phi=(\Phi_0,\dots,\Phi_{L-1})$ with envelope matrices $M_\ell$ (Section~\ref{sec:setting}) we study the \emph{factorisation stability complexity}
\[
\mathfrak B_{\mathrm{orb}}(\boldsymbol\Phi):=\inf_{d}\ \max_\ell\ \Lip_{d_\ell\to d_{\ell+1}}(\Phi_\ell),
\]
the best uniform layer constant obtainable by choosing weighted norms on the hidden spaces. For a discretisation operator $Q_h$ we study the \emph{budget-compatible approximation complexity}
\[
\mathfrak E_h(Q_h;B,K):=\sup\bigl\{\operatorname{osc}(Q_hf-f):\ \|f'\|_\infty\le B,\ \|f'''\|_\infty\le K\bigr\},
\]
the unavoidable error on a class constrained simultaneously in the norm the budget controls and in the norm that measures smoothness.

The second constraint is not decoration. With a single smoothness seminorm the lower-bound problem is degenerate: for $f_A(x)=Ax^2/2$ one has $f_A'''\equiv0$ for every $A$, so any operator that fails to reproduce quadratics exactly can be made to look arbitrarily bad by letting $A\to\infty$ --- while its Lipschitz constant blows up at the same rate, which is exactly what a layer budget forbids. Imposing $\|f'\|_\infty\le B$ removes this escape, and what remains is a genuine question about budget-preserving discretisation. This is why $\mathfrak E_h$ is defined on a double ball, and it is the point at which our lower bounds differ from classical saturation statements.

The two quantities are the two sides of one problem: a representation budget and a discretisation budget.

\subsection{Results}

\textbf{A. Exact finite-depth balancing} (Theorem~\ref{thm:balance}).
$\mathfrak B_{\mathrm{orb}}(\boldsymbol\Phi)=\|M_{L-1}\cdots M_0\|_{\infty\to\infty}^{1/L}$, with an explicit one-pass minimiser, for rectangular layers, together with the degenerate cases and an example of non-attainment. The proof rests on an exact formula for the layer Lipschitz constant in weighted $\ell^\infty$ norms (Proposition~\ref{prop:winf}), which is where the additive structure of the layers enters.

\textbf{B. Budget-compatible saturation} (Theorem~\ref{thm:budgetsat}).
For every linear spline-valued operator that contracts derivatives exactly,
\[
\mathfrak E_h(Q_h;B,K)\ \gtrsim\ \min\Bigl\{K_{\mathrm{eff}}|J|h^2,\ \frac{Bh^3}{|J|^2}\Bigr\},
\qquad K_{\mathrm{eff}}=\min\{K,8B/|J|^2\},
\]
for all $B,K>0$. The double constraint is what makes this a statement about budget-preserving discretisation rather than a classical saturation theorem: it blocks the scaling $f\mapsto Af_2$, which leaves the third derivative untouched while inflating the Lipschitz constant.

\textbf{C. Sharpness under composition} (Theorem~\ref{thm:universal}).
For every operator of the class there is a depth-$L$ tower inside the budget on which
$\|F-\widehat F\|\ge C(c)\sum_\ell\varepsilon_\ell$.
Together with the upper bound $\|F-\widehat F\|\le e^{c+\eta}\sum_\ell\varepsilon_\ell$ of Theorem~\ref{thm:Seta}, this pins the worst-case depth error between two constant multiples of $\sum_\ell\varepsilon_\ell$: linear accumulation in depth is not a proof artefact.

Theorem~\ref{thm:Seta} itself is the constructive half of the picture and is stated as such; the consequences for grid complexity, and an expressivity consequence of the residual form of the budget, are collected in Sections~\ref{sec:regimes} and \ref{sec:fold}.

\subsection{What is new and what is classical}

Our novelty is not the positive-operator moment argument, nor second-order saturation per se. The representation of a derivative-contracting operator as a positive averaging (Theorem~\ref{thm:T4a}) and the passage from vanishing first and second moments to a Dirac mass are the pointwise mechanism behind Korovkin's theorem; saturation of positive linear approximation processes at the second order is classical, and so is the exact layer formula once the additive structure is written down. Likewise, rescaling invariance of hidden units and diagonal balancing of a \emph{single} nonnegative matrix are long established.

Constrained approximation as such is of course classical --- shape-preserving, monotone, convex and Lipschitz-constrained approximation are large, old subjects, and so are lower estimates for operators with smooth range. We therefore state the two priority claims narrowly.

For Theorem~\ref{thm:balance}: unlike Perron-type scaling of a single matrix, Collatz--Wielandt quotients for pairs of operators, or multinorm certificates for constrained switching systems, we optimise here diagonal weighted $\ell^\infty$ norms along \emph{one} prescribed finite chain with \emph{both} endpoint norms fixed, and obtain the closed form
\[
\bigl\|M_{L-1}\cdots M_0\bigr\|_{\infty\to\infty}^{1/L}.
\] To the best of our knowledge this finite-horizon fixed-endpoint problem has not been solved in this form.

For Theorem~\ref{thm:budgetsat}: our claim concerns neither saturation of positive operators in general nor the Schoenberg operator in particular, both of which are classical. The new object is the minimax error over linear spline-valued derivative contractions when the adversarial function is constrained simultaneously in the Lipschitz seminorm controlled by the stability budget and in a third-derivative smoothness seminorm.

The classical tools --- the positive-averaging representation, the moment argument, second-order saturation --- are machinery for these statements, not claims of priority; see Remarks~\ref{rem:relwork} and \ref{rem:korovkin} for the comparison in detail.

\subsection{Relation to earlier work of the author}

The layer-wise Lipschitz product of a deep Kolmogorov--Arnold network was studied in \cite{Tankman}, where it is shown that any continuous function on $[0,1]^n$ representable by a finite computation tree with $N$ internal nodes and bounded compositional sparsity admits a deep KAN representation whose blocks have controlled depth and controlled block Lipschitz product, together with a domain-sensitive upper bound on the layer-wise product $\prod_\ell\max_iM_{\ell,i}$. That paper \emph{constructs} representations and \emph{bounds} the product from above.

The present paper starts where that one stops, and the two questions are complementary. Given the product, how should the stretching be distributed among the layers, and what is the smallest worst-layer constant attainable --- this is Theorem~\ref{thm:balance}, and its answer, $\|M_{L-1}\cdots M_0\|_{\infty\to\infty}^{1/L}$, is not the product of the individual layer constants but the $L$th root of the norm of the matrix product, which may be far smaller. And what does it cost to preserve such a budget when the univariate edges are discretised --- this is Theorem~\ref{thm:Seta} from above and Theorems~\ref{thm:budgetsat} and \ref{thm:universal} from below. Two further differences of technique are worth recording. We work throughout with the operator norm $\|M_\ell\|_{\infty\to\infty}$ of the envelope matrix rather than with an edgewise maximum, because Proposition~\ref{prop:winf} shows that the former \emph{is} the layer Lipschitz constant, exactly, whereas an edgewise bound does not control it once the width is taken into account. And the quantity that governs allocation turns out to be the envelope of the product, not the product of the envelopes: the gap between them can be arbitrarily large even when the network computes the identity (Corollary~\ref{cor:gap}). 

\section{Setting and conventions}
\label{sec:setting}

\subsection{Conventions fixed once}

\begin{itemize}[nosep]
\item $G$ is the number of \emph{intervals} of a uniform grid on a segment of length $2\rho$, and $h=2\rho/G$.
\item $k$ is the \emph{degree} of the piecewise polynomial.
\item $\Scal_{k,\Delta_h}$ denotes the space of \emph{continuous} piecewise polynomials of degree $\le k$ on the grid $\Delta_h$. Continuity is the only smoothness we assume; higher smoothness at the knots is never used, and $k=1$ (continuous piecewise linear) is included.
\item For a matrix $A$ we write $\opi{A}=\max_j\sum_i|A_{ji}|$.
\end{itemize}

\begin{definition}[KAN layer and network]
For $\Phi_\ell:\R^{n_\ell}\to\R^{n_{\ell+1}}$,
\[
(\Phi_\ell(x))_j=\sum_{i=1}^{n_\ell}\varphi_{\ell,j,i}(x_i),
\qquad
F=\Phi_{L-1}\circ\cdots\circ\Phi_0 .
\]
\end{definition}

\begin{definition}[Envelopes]
On a box $K=I_1\times\cdots\times I_n$ put
\[
\begin{gathered}
M_{ji}=\|\varphi'_{ji}\|_{L^\infty(I_i)},\qquad
H_{ji}=\|\varphi^{(k+1)}_{ji}\|_{L^\infty(I_i)},\\
\lambda=\opi{M},\qquad \kappa=\opi{H}.
\end{gathered}
\]
\end{definition}

We use repeatedly the identity
\begin{equation}\label{eq:rowsum}
A\ge 0\qquad\Longrightarrow\qquad \opi{A}=\max_j\sum_i A_{ji}=\|A\one\|_{\infty}.
\end{equation}

The input set is $K_0=[-R,R]^{n_0}$.

\subsection{Approximation operator}

We fix a linear operator $Q_{G,k}$ with the standard estimates
\begin{equation}\label{eq:Qest}
\|\varphi-Q\varphi\|_{\infty}\le C_{0,k}h^{k+1}\|\varphi^{(k+1)}\|_{\infty},
\qquad
\|\varphi'-(Q\varphi)'\|_{\infty}\le C_{1,k}h^{k}\|\varphi^{(k+1)}\|_{\infty} .
\end{equation}
Its anchored version is $Q^0_hf:=Q_hf-(Q_hf)(0)$; the derivative is unchanged and the $C^0$ constant at most doubles,
$\|f-Q^0_hf\|_{\infty}\le 2C_{0,k}h^{k+1}\|f^{(k+1)}\|_{\infty}$.

The class defined by \eqref{eq:Qest} is non-empty, with constants depending only on $k$; we record this because the phrase ``standard estimates'' is not by itself a proof, and because the boundary cells of a bounded interval are exactly where such statements are usually left vague.

\begin{lemma}[Polynomial halo extension]\label{lem:halo}
Let $I=[a,b]$, $r>0$ and $I^+:=[a-rh,b+rh]$. For $f\in C^{k+1}(I)$ define $E_hf:=f$ on $I$ and
\[
\begin{gathered}
E_hf(x):=\sum_{r'=0}^k\frac{f^{(r')}(a)}{r'!}(x-a)^{r'}\quad (x<a),\\
E_hf(x):=\sum_{r'=0}^k\frac{f^{(r')}(b)}{r'!}(x-b)^{r'}\quad (x>b).
\end{gathered}
\]
Then $E_h$ is linear, $E_hf\in C^{k}(I^+)\cap W^{k+1,\infty}(I^+)$,
\[
\|(E_hf)^{(k+1)}\|_{L^\infty(I^+)}\le\|f^{(k+1)}\|_{L^\infty(I)},
\qquad
E_hp=p\ \ \text{for every }p\in\Pcal_k .
\]
\end{lemma}

\begin{proof}
The Taylor polynomials match $f$ and its first $k$ derivatives at $a$ and $b$, so $E_hf\in C^k$; its $(k+1)$st derivative equals $f^{(k+1)}$ on $I$ and $0$ outside, which gives both displayed statements, the second because the Taylor polynomial of degree $k$ of a polynomial of degree $\le k$ is that polynomial.
\end{proof}

\begin{lemma}[Local $L^\infty$-bounded functionals]\label{lem:functionals}
Fix nodes $0\le s_0<\dots<s_k\le1$. For each $i$ let $J_i$ be the union of the $k+1$ cells starting at the left endpoint of $\supp N_{i,k}$, let $x_{i,0},\dots,x_{i,k}$ be the images of the $s_r$ under the affine map $[0,1]\to J_i$, and let $p_i[g]\in\Pcal_k$ be the Lagrange interpolant of $g$ at these nodes. Define
\[
\lambda_i(g):=c_i\bigl(p_i[g]\bigr),
\]
where $q=\sum_jc_j(q)N_{j,k}$ is the B-spline expansion of a polynomial $q\in\Pcal_k$, which exists and is unique by Marsden's identity \cite{Marsden}. Then $\lambda_i$ is linear, depends only on $g|_{J_i}$, and
\[
|\lambda_i(g)|\le\Lambda_k\kappa_k\,\|g\|_{L^\infty(J_i)},
\qquad
\tilde Qp=p\quad\text{for every }p\in\Pcal_k ,
\]
where $\Lambda_k$ is the Lebesgue constant of the reference nodes and $\kappa_k$ the norm of the B-spline coefficient functional $c_i$ on $\Pcal_k$ with respect to $\|\cdot\|_{L^\infty(J_i)}$; both are independent of $i$ and of $h$.
\end{lemma}

\begin{proof}
Locality and linearity are clear. For the bound, interpolation at the fixed reference nodes gives $\|p_i[g]\|_{L^\infty(J_i)}\le\Lambda_k\|g\|_{L^\infty(J_i)}$ with $\Lambda_k$ the Lebesgue constant, which is invariant under the affine rescaling $[0,1]\to J_i$. Next, $q\mapsto c_i(q)$ is a linear functional on the finite-dimensional space $\Pcal_k$, hence bounded with respect to $\|\cdot\|_{L^\infty(J_i)}$, say by $\kappa_k$; both $c_i(q)$ and that norm are invariant under an affine change of variable carrying the knots along, so $\kappa_k$ may be computed once on a reference configuration with $h=1$ and is independent of $h$. Composing the two bounds gives the claim. Finally, if $p\in\Pcal_k$ then $p_i[p]=p$ for every $i$, so $\tilde Qp=\sum_ic_i(p)N_{i,k}=p$.
\end{proof}

We emphasise that it is the $L^\infty$ bound that is needed here; this is why we do not use the de Boor--Fix functionals in their original form \cite{deBoorFix}, which are expressed through derivatives and are bounded on $C^k$ rather than on $L^\infty$. Constructions of local, $L^\infty$-bounded spline approximation operators are of course classical, see \cite[Ch.~6]{Schumaker} and \cite[Ch.~XII]{deBoorPGS}; we have given one explicitly so that Section~\ref{sec:setting} depends on no external estimate.

\begin{lemma}[Existence]\label{lem:existence}
For every $k$ there is a family of linear operators $Q_{G,k}:C^{k+1}(I)\to\Scal_{k,\Delta_h}$, $I$ a bounded interval, satisfying \eqref{eq:Qest} with $C_{0,k},C_{1,k}$ depending only on $k$.
\end{lemma}

\begin{proof}
Let $\tilde\Delta_h$ be the uniform grid of spacing $h$ extending $\Delta_h$, and let $\tilde Qg=\sum_i\lambda_i(g)N_{i,k}$ with the functionals $\lambda_i$ of Lemma~\ref{lem:functionals}. Only those $i$ with $N_{i,k}\not\equiv0$ on $I$ matter, and the corresponding $\lambda_i$ see only the finite halo $I^+=[a-r_kh,b+r_kh]$. Set $Q_{G,k}f:=(\tilde QE_hf)|_I$ with $E_h$ from Lemma~\ref{lem:halo}; this is linear and spline-valued, and no extension theorem is invoked.

Fix a cell $I_i=[t_i,t_{i+1}]\subset I$. The B-splines that do not vanish on $I_i$ are $N_{j,k}$ for $j=i-k,\dots,i$, and $\lambda_j$ sees only $J_j=[t_j,t_{j+k+1}]$, so the relevant region is
\[
U_i=\bigcup_{j=i-k}^{i}J_j=[t_{i-k},t_{i+k+1}],
\qquad |U_i|=(2k+1)h ,
\]
which also fixes the halo width to $r_k=k$. Let $p\in\Pcal_k$ be the Taylor polynomial of $E_hf$ of degree $k$ at the centre of $U_i$; by Taylor's theorem applied to $E_hf\in W^{k+1,\infty}(U_i)$ and Lemma~\ref{lem:halo},
\[
\begin{aligned}
\|E_hf-p\|_{L^\infty(U_i)}&\le\frac{(R_kh)^{k+1}}{(k+1)!}\|f^{(k+1)}\|_{L^\infty(I)},\\
\|(E_hf-p)'\|_{L^\infty(U_i)}&\le\frac{(R_kh)^{k}}{k!}\|f^{(k+1)}\|_{L^\infty(I)},
\end{aligned}
\qquad R_k:=\frac{2k+1}{2},
\]
Since $\tilde Qp=p$ and $\sum_jN_{j,k}\equiv1$,
\[
\begin{aligned}
\|Q_{G,k}f-f\|_{L^\infty(I_i)}&=\|\tilde Q(E_hf-p)-(E_hf-p)\|_{L^\infty(I_i)}\\
&\le(1+\Lambda_k\kappa_k)\|E_hf-p\|_{L^\infty(U_i)}\ \le\ C_{0,k}h^{k+1}\|f^{(k+1)}\|_{L^\infty(I)} .
\end{aligned}
\]
For the derivative, $\tilde Q(E_hf-p)$ restricted to $I_i$ is a polynomial of degree $\le k$ on an interval of length $h$, so Markov's inequality gives $\|(\tilde Q(E_hf-p))'\|_{L^\infty(I_i)}\le2k^2h^{-1}\|\tilde Q(E_hf-p)\|_{L^\infty(I_i)}$; combining with the two displays above,
\[
\|(Q_{G,k}f-f)'\|_{L^\infty(I_i)}\le C_{1,k}h^{k}\|f^{(k+1)}\|_{L^\infty(I)} .
\]
All constants depend only on $k$ and no cell is distinguished, so the estimates hold uniformly up to the ends of $I$.
\end{proof}

Two remarks on the route chosen. First, the halo extension is local, elementary and polynomial-preserving; it is exactly the property $E_hp=p$ for $p\in\Pcal_k$ that the quasi-interpolant needs, and it avoids any appeal to a ``usual modification near the ends''. Second, a global Sobolev extension theorem is \emph{not} available in the form one might be tempted to use: no linear $E$ can satisfy $\|(Ef)^{(j)}\|_{L^\infty(\R)}\le c\|f^{(j)}\|_{L^\infty(I)}$ for every $j\le k+1$ simultaneously, since $f(x)=x$ has $\|f''\|_\infty=0$, forcing $Ef$ to be affine and non-constant, hence unbounded, in contradiction with the case $j=0$. Classical extension theorems \cite{Stein} control the full $W^{k+1,\infty}$ norm, not each seminorm by the corresponding seminorm. Everything else in Sections~\ref{sec:Seta} and \ref{sec:regimes} is conditional only on \eqref{eq:Qest}.

\subsection{Centering}

Imposing $\varphi_{\ell,ji}(0)=0$ edgewise would force $F(0)=0$. This is avoided without introducing biases.

\begin{lemma}[Centering]\label{lem:centering}
Let $z_0=0$ and $z_{\ell+1}=\Phi_\ell(z_\ell)$ be the reference trajectory of the exact network. Put
\[
\psi_{\ell,ji}(u)=\varphi_{\ell,ji}(z_{\ell,i}+u)-\varphi_{\ell,ji}(z_{\ell,i}),
\qquad
(\Psi_\ell(u))_j=\sum_i\psi_{\ell,ji}(u_i).
\]
Then $\psi_{\ell,ji}(0)=0$, $\Psi_\ell(0)=0$,
\[
\Phi_\ell(z_\ell+u)-z_{\ell+1}=\Psi_\ell(u),
\qquad
F(x)-z_L=\Psi_{L-1}\circ\cdots\circ\Psi_0(x),
\]
and $\psi'_{\ell,ji}(u)=\varphi'_{\ell,ji}(z_{\ell,i}+u)$, so all Lipschitz and curvature envelopes are preserved after translating the corresponding intervals.
\end{lemma}

\begin{proof}
Immediate; the only point to check is that $\Psi_\ell$ is again a KAN layer, which holds because the translation acts coordinatewise on the univariate arguments.
\end{proof}

\begin{remark}\label{rem:anchor}
(i) Anchoring pins the reference trajectory of \emph{both} networks: $\widehat\psi(0)=0$ gives $\widehat F(0)=F(0)=z_L$, so the discretisation is exact along the anchor and the layer errors are not uniform over the domain. This is a normalisation of the operator, not a loss of generality.
(ii) $z_\ell$ is an object of the exact network, so the domain $z_\ell+[-\rho,\rho]^{n_\ell}$ depends on $\Phi$; the quantifier order is: first $c,\eta,R$, then $\rho$, then $z_\ell$, then the hypotheses on the translated box.
(iii) If $Q_h$ does not preserve the value at $0$, use $Q^0_h$. Anchoring is used in Sections~\ref{sec:winf}--\ref{sec:Seta}; it is deliberately \emph{not} imposed in Section~\ref{sec:sharp}, where one-sidedness of the spline error would be destroyed by subtracting $(Q_hf)(0)$, and the trajectories are controlled directly instead.
\end{remark}

\subsection{Gauge}

\begin{definition}[Gauge]\label{def:gauge}
A \emph{gauge} is a choice of weighted norms on the hidden spaces: for $d_\ell\in(0,\infty)^{n_\ell}$ with $d_0=d_L=\one$ and $D_\ell=\diag(d_\ell)$,
\[
\|u\|_{d_\ell}:=\max_i\frac{|u_i|}{d_{\ell,i}} .
\]
The network, its coordinates, its boxes and its grids are \emph{not} changed.
\end{definition}

\begin{remark}[Realisation as a reparametrisation]\label{rem:reparam}
The same algebra can be realised as an actual rescaling of the network,
\[
\widetilde\varphi_{\ell,j,i}(t)=\frac1{d_{\ell+1,j}}\varphi_{\ell,j,i}(d_{\ell,i}t),
\]
which preserves the univariate structure, the normalisation $\varphi(0)=0$ and the function $F$. However, the identity $\widetilde M_\ell=D_{\ell+1}^{-1}M_\ell D_\ell$ then holds only if the coordinate boxes are transformed accordingly, $I'_{\ell,i}=d_{\ell,i}^{-1}I_{\ell,i}$: the supremum of $\widetilde\varphi'$ over the \emph{same fixed} interval need not equal $d_{\ell,i}d_{\ell+1,j}^{-1}M_{\ell,ji}$. Theorem~\ref{thm:balance} is stated purely in terms of Definition~\ref{def:gauge}, where no such issue arises; the reparametrisation reading is used nowhere in the proofs.
\end{remark}

\section{The exact layer Lipschitz constant}
\label{sec:winf}

\begin{proposition}[\texorpdfstring{$W_\infty$}{W-infinity}]\label{prop:winf}
Let $K=I_1\times\cdots\times I_n$ be a box with nondegenerate segments and $\varphi_{ji}\in W^{1,\infty}(I_i)$. Then, with the weighted norms of Definition~\ref{def:gauge},
\[
\Lip_{d\to d'}(\Phi;K)=\max_j\frac{1}{d'_j}\sum_i M_{ji}d_i=\opi{D'^{-1}MD}.
\]
In particular $\Lip_{\infty\to\infty}(\Phi;K)=\opi{M}$.
\end{proposition}

\begin{proof}
\emph{Upper bound.} For each $j$,
\[
\begin{aligned}
|(\Phi x)_j-(\Phi y)_j|&\le\sum_i|\varphi_{ji}(x_i)-\varphi_{ji}(y_i)|\le\sum_iM_{ji}|x_i-y_i|\\
&\le\Bigl(\sum_iM_{ji}d_i\Bigr)\|x-y\|_{d},
\end{aligned}
\]
and division by $d'_j$ gives the bound.

\emph{Lower bound.} Fix $j$ attaining the maximum and put $I_+=\{i:M_{ji}>0\}$. If $I_+=\varnothing$ the row sum is $0$ and there is nothing to prove, so assume $I_+\ne\varnothing$ and fix
\[
0<\varepsilon<\tfrac13\min_{i\in I_+}M_{ji}.
\]
For $i\in I_+$ the set $\{t\in\operatorname{int}I_i:\ |\varphi'_{ji}(t)|>M_{ji}-\varepsilon\}$ has positive measure; pick in it a Lebesgue point $t_i$ of $\varphi'_{ji}$ and set $\sigma_i=\operatorname{sign}\varphi'_{ji}(t_i)$. For $i\notin I_+$ put $t_i$ to be any interior point and $\sigma_i=0$. By the Lebesgue point property, for $i\in I_+$,
\[
\varphi_{ji}(t_i+\sigma_i\delta d_i)-\varphi_{ji}(t_i)
=\int_{t_i}^{t_i+\sigma_i\delta d_i}\varphi'_{ji}
=\sigma_i\delta d_i\bigl(\varphi'_{ji}(t_i)+o(1)\bigr),\qquad\delta\to0^+,
\]
so the increment is at least $\delta d_i(M_{ji}-2\varepsilon)>0$ for all small $\delta$; the strict positivity uses $\varepsilon<M_{ji}/3$. Since $I_+$ is finite, one may choose a single $\delta>0$ that works for all $i\in I_+$ and is small enough that $t_i+\sigma_i\delta d_i\in I_i$ for every $i$. Put $x=(t_i)_i$ and $y=(t_i+\sigma_i\delta d_i)_i$. Then $\|x-y\|_d=\delta$ (the maximum being attained on $I_+$), the summands with $i\notin I_+$ vanish, and all the remaining ones have the same sign, hence add up:
\[
\frac{|(\Phi y)_j-(\Phi x)_j|}{d'_j}\ \ge\ \frac{\delta}{d'_j}\Bigl(\sum_{i\in I_+}M_{ji}d_i-2\varepsilon\sum_{i\in I_+}d_i\Bigr).
\]
Since $\sum_{i\in I_+}M_{ji}d_i=\sum_iM_{ji}d_i$, letting $\varepsilon\downarrow0$ gives the claim.
\end{proof}

The key point is the independence of coordinates, i.e.\ the box structure.

\begin{remark}[The box is essential]\label{rem:box}
On a reachable set that is not a box the equality fails. Take $n=2$, $K=\{x_1=x_2\}\cap[-1,1]^2$, $\varphi_{11}(t)=t$, $\varphi_{12}(t)=-t$. Then $\opi{M}=2$ while $\Phi|_K\equiv0$. Exactness in $\ell^\infty$ is bought at the price of the most conservative domain.
\end{remark}

\begin{corollary}[Exact weighted row balancing of a single layer]\label{cor:N}
Given $d_\ell$, define $d_{\ell+1,j}:=\sum_iM_{\ell,ji}d_{\ell,i}$ for those $j$ with a positive row sum, and $d_{\ell+1,j}:=1$ otherwise. Then
\[
\opi{D_{\ell+1}^{-1}M_\ell D_\ell}=1
\]
exactly, unless $M_\ell=0$, in which case it is $0$. The computation is one pass and requires no iteration.
\end{corollary}

\begin{proof}
Immediate from \eqref{eq:rowsum} and Proposition~\ref{prop:winf}: rows with positive sum are normalised to $1$, rows that vanish contribute $0$.
\end{proof}

We deliberately do not call this an ``$\ell^\infty$ analogue of spectral normalisation'': it is exact weighted row balancing of one layer, and Theorem~\ref{thm:balance} shows that balancing all layers simultaneously is a different problem. A greedy left-to-right pass makes $\lambda_\ell=1$ for $\ell\le L-2$, but since $d_L=\one$ the accumulated scale is deposited in the last layer; Theorem~\ref{thm:balance} explains why.

\section{Exact gauge balancing}
\label{sec:balance}

Throughout this section $M_\ell\in\R^{n_{\ell+1}\times n_\ell}_{\ge0}$,
\[
P:=M_{L-1}\cdots M_0,\qquad \tau_*:=\opi{P}^{1/L},\qquad
u_0=\one,\quad u_{\ell+1}=M_\ell u_\ell ,
\]
so $u_L=P\one$ and $\tau_*=\|u_L\|_\infty^{1/L}$. The balancing functional is
\[
\Bo(M_0,\dots,M_{L-1}):=\inf_{\substack{d_1,\dots,d_{L-1}>0\\ d_0=d_L=\one}}\ \max_{0\le\ell<L}\opi{D_{\ell+1}^{-1}M_\ell D_\ell}.
\]
By Proposition~\ref{prop:winf} this is the same quantity as the factorisation stability complexity $\Bo(\boldsymbol\Phi)$ of Section~\ref{sec:intro}, and we use the two notations interchangeably, writing the matrices when the argument is a chain of envelopes and the factorisation when it is a network.

\begin{theorem}[Exact orbit balancing]\label{thm:balance}
$\Bo=\opi{M_{L-1}\cdots M_0}^{1/L}$.
If moreover $u_\ell>0$ coordinatewise for $\ell=1,\dots,L-1$ and $\tau_*>0$, the infimum is attained at
$d_\ell=\tau_*^{-\ell}u_\ell$, $\ell=1,\dots,L-1$.
\end{theorem}

\begin{proof}[Proof of the lower bound]
All $M_\ell\ge0$ and $D_\ell>0$, so the product telescopes exactly:
\[
\prod_{\ell=L-1}^{0}\bigl(D_{\ell+1}^{-1}M_\ell D_\ell\bigr)=D_L^{-1}PD_0=P .
\]
By submultiplicativity and the inequality between the geometric mean and the maximum,
\[
\tau_*^L=\opi{P}\le\prod_\ell\opi{D_{\ell+1}^{-1}M_\ell D_\ell}\le\Bigl(\max_\ell\opi{D_{\ell+1}^{-1}M_\ell D_\ell}\Bigr)^L .
\]
Taking the infimum over $d$ gives $\Bo\ge\tau_*$.
\end{proof}

\begin{proof}[Proof of the upper bound, nondegenerate case]
Let $u_\ell>0$ for $1\le\ell\le L-1$ and $\tau_*>0$; set $d_\ell=\tau_*^{-\ell}u_\ell$.
For $0\le\ell\le L-2$ we have $M_\ell d_\ell=\tau_*^{-\ell}u_{\ell+1}=\tau_*d_{\ell+1}$, hence
\[
\bigl(D_{\ell+1}^{-1}M_\ell D_\ell\bigr)\one=D_{\ell+1}^{-1}M_\ell d_\ell=\tau_*\one
\ \overset{\eqref{eq:rowsum}}{\Longrightarrow}\
\opi{D_{\ell+1}^{-1}M_\ell D_\ell}=\tau_*\quad\text{exactly.}
\]
For $\ell=L-1$, where $D_L=I$,
\[
M_{L-1}d_{L-1}=\tau_*^{-(L-1)}u_L\le\tau_*^{-(L-1)}\|u_L\|_\infty\one=\tau_*\one,
\]
so $\opi{M_{L-1}D_{L-1}}\le\tau_*$. Hence $\max_\ell=\tau_*$.
\end{proof}

The degenerate cases are treated in Appendix~\ref{sec:degenerate}: the value is still $\tau_*$, but the infimum need not be attained.

\begin{remark}[Structure of the optimum]
All layers except the last are balanced to \emph{equality}, the last only to an inequality. This explains the behaviour of the greedy normalisation of Corollary~\ref{cor:N}: the greedy pass computes the same vectors $u_\ell$ but without the factor $\tau_*^{-\ell}$.
\end{remark}

\subsection{Corollaries}

\begin{corollary}[Budget criterion]\label{cor:criterion}
Let $b:=1+c/L$. Then
\[
\tau_*<b\ \Longrightarrow\ \exists D:\ \max_\ell\opi{D_{\ell+1}^{-1}M_\ell D_\ell}<b,
\qquad
\tau_*>b\ \Longrightarrow\ \text{infeasible},
\]
and for $\tau_*=b$ feasibility is equivalent to attainment of the infimum. In the nondegenerate case the criterion is an equivalence:
\[
\exists D:\ \max_\ell\lambda^{(d)}_\ell\le b
\quad\Longleftrightarrow\quad
\opi{{\textstyle\prod_\ell} M_\ell}\le b^{L}\ (\le e^c).
\]
\end{corollary}

\begin{proof}
The first implication uses that values arbitrarily close to $\tau_*$ are achieved, degenerate or not (Theorem~\ref{thm:balance} and Appendix~\ref{sec:degenerate}); the second is the lower bound; the third is attainment.
\end{proof}

\begin{remark}
The equivalence must \emph{not} be stated in general: Example~\ref{ex:nonattain} has $\tau_*=1$ with the infimum not attained, so $\tau_*\le b$ with $b=1$ does not imply existence of $D$.
\end{remark}

\begin{corollary}[Closed-form algorithm]\label{cor:algo}
One forward pass $u_0=\one$, $u_{\ell+1}=M_\ell u_\ell$ produces $\tau_*=\|u_L\|_\infty^{1/L}$ and the family $d_\ell=\tau_*^{-\ell}u_\ell$, at cost $O(\sum_\ell\mathrm{nnz}(M_\ell))$ and without iteration. Its status: a \emph{minimiser} if $u_\ell>0$ and $\tau_*>0$; an $\varepsilon$-optimal gauge after the regularisation of Appendix~\ref{sec:degenerate} if derivative-dead coordinates are present; and an infimum statement only, if $\tau_*=0$.
\end{corollary}

\subsection{Robustness of the optimal budget and sensitivity of the gauge}

In practice the $M_\ell$ are estimated numerically from trained edges, so the exact formula of Theorem~\ref{thm:balance} is useful only if it is stable. It is, and with no loss over depth; the optimal gauge, however, need not be.

\begin{proposition}[Multiplicative perturbation]\label{prop:perturb}
Let $M_\ell,\widetilde M_\ell\ge0$ satisfy, entrywise,
\[
(1-\varepsilon_\ell)M_\ell\ \le\ \widetilde M_\ell\ \le\ (1+\varepsilon_\ell)M_\ell,
\qquad 0\le\varepsilon_\ell<1 .
\]
Then, with $\widetilde\tau_*=\|\widetilde M_{L-1}\cdots\widetilde M_0\|_{\infty\to\infty}^{1/L}$,
\[
\tau_*\Bigl(\prod_\ell(1-\varepsilon_\ell)\Bigr)^{1/L}\ \le\ \widetilde\tau_*\ \le\ \tau_*\Bigl(\prod_\ell(1+\varepsilon_\ell)\Bigr)^{1/L},
\]
and in particular $(1-\varepsilon)\tau_*\le\widetilde\tau_*\le(1+\varepsilon)\tau_*$ if $\varepsilon_\ell\le\varepsilon$ for all $\ell$.
\end{proposition}

\begin{proof}
All factors are nonnegative, so entrywise inequalities multiply:
$\bigl(\prod_\ell(1-\varepsilon_\ell)\bigr)P\le\widetilde P\le\bigl(\prod_\ell(1+\varepsilon_\ell)\bigr)P$.
The norm $\|\cdot\|_{\infty\to\infty}$ is monotone with respect to the entrywise order on nonnegative matrices, so the same inequalities hold for the norms; take $L$th roots.
\end{proof}

Thus the optimal layer budget inherits the \emph{relative} accuracy of the envelopes, degraded by no factor depending on $L$: the geometric mean in the statement is what prevents accumulation. This is what makes Corollary~\ref{cor:algo} a normalisation principle rather than an algebraic identity.

\begin{proposition}[Sensitivity of the optimal gauge]\label{prop:gauge-sens}
In the nondegenerate case, with $A_\pm:=\prod_{r<L}(1\pm\varepsilon_r)$ and $B_{\ell,\pm}:=\prod_{r<\ell}(1\pm\varepsilon_r)$, the optimal gauges satisfy, coordinatewise,
\[
\frac{B_{\ell,-}}{A_+^{\ell/L}}\ \le\ \frac{\widetilde d_{\ell,i}}{d_{\ell,i}}\ \le\ \frac{B_{\ell,+}}{A_-^{\ell/L}} ,
\qquad\text{so for }\varepsilon_\ell\le\varepsilon:\quad
\Bigl(\frac{1-\varepsilon}{1+\varepsilon}\Bigr)^{\ell}\le\frac{\widetilde d_{\ell,i}}{d_{\ell,i}}\le\Bigl(\frac{1+\varepsilon}{1-\varepsilon}\Bigr)^{\ell}.
\]
Both uniform-$\varepsilon$ bounds are sharp.
\end{proposition}

\begin{proof}
The same entrywise multiplication gives $B_{\ell,-}u_\ell\le\widetilde u_\ell\le B_{\ell,+}u_\ell$, while
\[
\frac{\widetilde d_\ell}{d_\ell}=\Bigl(\frac{\tau_*}{\widetilde\tau_*}\Bigr)^{\ell}\,\frac{\widetilde u_\ell}{u_\ell} ,
\]
and Proposition~\ref{prop:perturb} bounds the first factor by $A_\mp^{-\ell/L}$. For attainment take $n_\ell\equiv2$, $M_\ell=I$ and $\widetilde M_\ell=\diag(1+\varepsilon,1-\varepsilon)$: then $\tau_*=1$, $\widetilde\tau_*=1+\varepsilon$, $d_\ell=(1,1)$ and $\widetilde d_\ell=\bigl(1,((1-\varepsilon)/(1+\varepsilon))^{\ell}\bigr)$, so the lower bound holds with equality in the second coordinate, and the upper bound with equality after exchanging the two diagonal entries.
\end{proof}

\begin{remark}[Well-conditioned value, ill-conditioned gauge]\label{rem:cond}
Propositions~\ref{prop:perturb} and \ref{prop:gauge-sens} say different things, and the example in the proof shows the difference is real rather than an artefact of the estimates: the optimal \emph{value} is stable uniformly in depth, while the optimal \emph{weights} can drift exponentially in $\ell$. For a practitioner this means the number $\tau_*$ may be trusted from estimated envelopes, whereas a gauge computed from them should be recomputed rather than transported along a long chain.
\end{remark}

\begin{remark}[Additive perturbations]\label{rem:additive}
The multiplicative hypothesis fixes the zero pattern of the envelopes. If instead $\widetilde M_\ell=M_\ell+\Delta_\ell$ with $\Delta_\ell$ of arbitrary sign pattern, telescoping gives
\[
\|\widetilde P-P\|_{\infty\to\infty}\le\sum_{j<L}\Bigl(\prod_{r>j}\|\widetilde M_r\|_{\infty\to\infty}\Bigr)\|\Delta_j\|_{\infty\to\infty}\Bigl(\prod_{r<j}\|M_r\|_{\infty\to\infty}\Bigr),
\]
and hence $|\widetilde\tau_*^L-\tau_*^L|\le\|\widetilde P-P\|_{\infty\to\infty}$. This is weaker: the bound now degrades with depth through the products of layer norms, which is why we state the multiplicative version as the theorem.
\end{remark}

\begin{corollary}[Cancellation gap]\label{cor:gap}
Linear edges form a legitimate KAN. Let
\[
W_0=\begin{pmatrix}1&K\\0&1\end{pmatrix},\qquad
W_1=\begin{pmatrix}1&-K\\0&1\end{pmatrix},\qquad W_1W_0=I .
\]
Then $\Lip_\infty(F)=1$, while $M_0=M_1=\begin{psmallmatrix}1&K\\0&1\end{psmallmatrix}$, $P=\begin{psmallmatrix}1&2K\\0&1\end{psmallmatrix}$, and
$\Bo=\sqrt{1+2K}\to\infty$ as $K\to\infty$.
\end{corollary}

\begin{proof}
$\opi{P}=1+2K$, so $\tau_*=\sqrt{1+2K}$; attainment is explicit with $u_1=(1+K,1)^\top$ and $d_1=\tau_*^{-1}u_1$, both layer norms being then equal to $\tau_*$.
\end{proof}

The mechanism is that passing to envelopes destroys signs: $|W_1|\,|W_0|\ne|W_1W_0|$, and no diagonal recalibration can restore them.

\begin{corollary}[Homogeneous tower]\label{cor:perron}
If $M_\ell\equiv M$ is square, then by Gelfand's formula
\[
\lim_{L\to\infty}\Bo=\lim_{L\to\infty}\opi{M^L}^{1/L}=\rho(M).
\]
Hence $\rho(M)>1$ makes a fixed-$c$ budget asymptotically infeasible, $\rho(M)<1$ makes it feasible, and $\rho(M)=1$ is decided by the growth of $\opi{M^L}$, not by $\rho$ alone.
\end{corollary}

\begin{example}[The boundary case is real]
For $M=\begin{psmallmatrix}1&K\\0&1\end{psmallmatrix}$ one has $\opi{M^L}=1+LK$, hence
$\Bo=(1+LK)^{1/L}=1+L^{-1}\log(LK)+o(L^{-1})$,
which does not fit into $1+c/L$ for any fixed $c$, although $\rho(M)=1$.
\end{example}

\begin{remark}[Related work and the status of the novelty claim]\label{rem:relwork}
For a \emph{single} square irreducible nonnegative $M$, the identity $\inf_D\opi{D^{-1}MD}=\rho(M)$, with the optimal $D$ built from the Perron vector, is classical \cite{Bauer1963,Osborne1960}; iterative balancing in other norms is the Osborne--Parlett--Reinsch family. In the neural-network literature, rescaling-invariant reparametrisations of hidden units and the associated path-based quantities are studied in \cite{PathSGD}, and a different balancing objective ($\ell^2$ weight norm within a rescaling class, solved iteratively) in \cite{EquiNorm}. Theorem~\ref{thm:balance} is the \emph{layered, inhomogeneous, rectangular} problem induced by weighted $\ell^\infty$ KAN Lipschitz constants, with an exact value, a closed-form minimiser and a complete treatment of degeneracies. Three further lines deserve mention, because each is close in spirit.

First, Collatz--Wielandt theory has been developed beyond a single square matrix: Friedland \cite{Friedland} studies the Collatz--Wielandt quotient for \emph{pairs} of nonnegative operators between cones, including rectangular nonnegative matrices. The object there is a quotient attached to two operators, not a minimax along a prescribed chain with intermediate norms and fixed endpoints.

Second, nonlinear Perron--Frobenius theory for order-preserving multi-homogeneous maps \cite{PFmulti} proves Collatz--Wielandt principles for the corresponding nonlinear spectral radius, and the recursion $u_{\ell+1}=M_\ell u_\ell$ of Theorem~\ref{thm:balance} is recognisably of that family. That theory answers an eigenvalue question; ours is a boundary-value optimisation along a finite chain, with no asymptotic eigenvector.

Third, and closest of all in form, the \emph{multinorm} framework for constrained switching systems assigns different norms to the nodes of an automaton and requires every transition to be contractive, which characterises the constrained joint spectral radius \cite{Philippe}. The comparison is instructive. Multinorm and path-complete constructions control \emph{all} admissible products of a family and their asymptotic growth, and the endpoint norms are part of the unknown certificate; our problem concerns \emph{one} prescribed finite chain, restricts the intermediate norms to diagonal weighted $\ell^\infty$ norms, and fixes both endpoint norms, $d_0=d_L=\one$. It is precisely this fixing of both ends that makes the finite-chain minimax nontrivial and gives it a closed form. Relatedly, the joint spectral radius of a family is an infimum over norms of a supremum over all products and is hard to compute, whereas here the product is a single fixed ordered chain.

To the best of our knowledge, the finite-horizon fixed-endpoint problem solved in Theorem~\ref{thm:balance}, together with its treatment of rectangular chains, degeneracies and perturbations, has not appeared previously.
\end{remark}

\section{A constructive upper bound: stable spline discretisation}
\label{sec:Seta}

\begin{theorem}[\texorpdfstring{$S_\eta$}{S-eta}]\label{thm:Seta}
Fix $c,\eta,R>0$, $L$ and $k$, and set $\rho:=e^{c+\eta}R$, $h=2\rho/G$. Assume, on the boxes $z_\ell+[-\rho,\rho]^{n_\ell}$ and after centering (Lemma~\ref{lem:centering}):
\begin{enumerate}[label=\textup{(H\arabic*)},nosep]
\item $\psi_{\ell,ji}\in C^{k+1}$, $\psi_{\ell,ji}(0)=0$, $\widehat\psi_{\ell,ji}(0)=0$;
\item $\lambda_\ell\le 1+c/L$ for all $\ell$;
\item the estimates \eqref{eq:Qest} for $Q_{G,k}$;
\item $\sum_\ell C_{1,k}\kappa_\ell h^k\le\eta$.
\end{enumerate}
Then both trajectories issued from $K_0$ stay inside those boxes, and
\[
\Lip_\infty(\widehat F;K_0)\le e^{c+\eta},
\qquad
\|F-\widehat F\|_{C^0(K_0)}\le 2C_{0,k}\,e^{c+\eta}h^{k+1}\sum_\ell\kappa_\ell .
\]
\end{theorem}

\begin{proof}
Write $A_\ell=C_{1,k}\kappa_\ell h^k$.

\emph{Step 1 (inheritance).} Entrywise $\widehat M_{\ell,ji}\le M_{\ell,ji}+C_{1,k}h^kH_{\ell,ji}$, so by Proposition~\ref{prop:winf}
\[
\widehat\lambda_\ell=\opi{\widehat M_\ell}\le\lambda_\ell+C_{1,k}h^k\kappa_\ell=\lambda_\ell+A_\ell\le\Bigl(1+\frac cL\Bigr)(1+A_\ell).
\]
Note that $\kappa_\ell=\opi{H_\ell}$ is exactly the functional produced by $\max_j\sum_i$.

\emph{Step 2 (Lipschitz bound).} $\prod_\ell\widehat\lambda_\ell\le e^ce^{\sum_\ell A_\ell}\le e^{c+\eta}$ by (H4). Since the boxes are convex and, by Step 3, contain the images of $K_0$ under all partial compositions, the layer bounds compose, giving $\Lip_\infty(\widehat F;K_0)\le e^{c+\eta}$.

\emph{Step 3 (localisation).} Since $\widehat\Psi_\ell(0)=0$ we have $\|\hat u_{\ell+1}\|_\infty\le\widehat\lambda_\ell\|\hat u_\ell\|_\infty$, hence by induction $\|\hat u_\ell\|_\infty\le e^{c+\eta}R=\rho$ for all $\ell\le L$; the same holds for the exact trajectory, with $e^c$. There is no circularity: $\rho$ is fixed first and (H2), (H4) are hypotheses \emph{on} the boxes of radius $\rho$.

\emph{Step 4 (error).} With $e_\ell=\hat u_\ell-u_\ell$,
\[
e_{\ell+1}=\bigl[\widehat\Psi_\ell(\hat u_\ell)-\widehat\Psi_\ell(u_\ell)\bigr]+\bigl[(\widehat\Psi_\ell-\Psi_\ell)(u_\ell)\bigr],
\qquad
\|e_{\ell+1}\|_\infty\le\widehat\lambda_\ell\|e_\ell\|_\infty+\varepsilon_\ell,
\]
with $\varepsilon_\ell\le2C_{0,k}h^{k+1}\kappa_\ell$; both arguments lie in the boxes by Step 3. Unrolling,
$\|e_L\|_\infty\le\sum_\ell(\prod_{m>\ell}\widehat\lambda_m)\varepsilon_\ell\le e^{c+\eta}\sum_\ell\varepsilon_\ell$.
\end{proof}

\begin{remark}
The conclusion is local by construction: all hypotheses live on boxes of radius $\rho$ around the reference trajectory, so only $\Lip(\widehat F;K_0)$ is proved. A global statement would require a global extension of the splines and does not follow from (H1)--(H4).
\end{remark}

\begin{corollary}[Grid threshold]\label{cor:threshold}
If $\kappa_\ell\le\kappa$ then (H4) reads
$G\ge 2e^{c+\eta}R\bigl(C_{1,k}\kappa L/\eta\bigr)^{1/k}$,
and then $\|F-\widehat F\|_{C^0(K_0)}\le 2C_{0,k}e^{c+\eta}\kappa Lh^{k+1}$.
\end{corollary}

\begin{corollary}[Where the width enters]
An absolute curvature bound $\|H_{\ell,ji}\|_\infty\le B$ gives $\kappa_\ell\le n_\ell B$; a rowwise relative bound $\kappa_\ell\le B\lambda_\ell$ gives $\kappa_\ell\le B(1+c/L)$. Both classes are legitimate: small-amplitude high-frequency edges are a genuine object, and the factor $n$ in the first class is real, not an artefact of normalisation.
\end{corollary}

Theorem~\ref{thm:Seta} is a sufficiency statement for the chosen $Q_{G,k}$, not a necessity statement; see Section~\ref{sec:regimes} and Open problem~\ref{op:overshoot}.

\begin{remark}[Anisotropic transfer]\label{rem:aniso}
In weighted norms the correct objects, for $h_{\ell,i}=2\rho d_{\ell,i}/G_{\ell,i}$, are
\[
\begin{gathered}
A^{(d)}_\ell=C_{1,k}\max_j\frac1{d_{\ell+1,j}}\sum_i d_{\ell,i}H_{\ell,ji}h_{\ell,i}^{\,k},\\
\varepsilon^{(d)}_\ell=2C_{0,k}\max_j\frac1{d_{\ell+1,j}}\sum_i H_{\ell,ji}h_{\ell,i}^{\,k+1}.
\end{gathered}
\]
The factor $d_{\ell,i}$ appears in $A^{(d)}$ because the derivative is taken in the weighted coordinate, and is absent from $\varepsilon^{(d)}$ because the $C^0$ error is absolute in the output coordinate and is only divided by $d_{\ell+1,j}$. Weighted norms thus induce coordinate-dependent spline domains and permit adaptive anisotropic grid allocation; optimising $G_{\ell,i}$ under $\sum_iG_{\ell,i}=G_\ell$ is a separate step, not carried out here, and no claim that anisotropy is strictly better is made before it.
\end{remark}

\section{Lower bounds for budget-compatible spline approximation}
\label{sec:rigidity}

Throughout this section $Q_h:C^1(J)\to\Scal_{k,\Delta_h}$ is linear and reproduces affine functions, and \emph{Lip-diminishing} means
\begin{equation}\label{eq:lipdim}
\|(Q_hf)'\|_{L^\infty}\le\|f'\|_{L^\infty}\qquad\text{for all }f\in C^1(J).
\end{equation}
Since elements of $\Scal_{k,\Delta_h}$ are only continuous, $(Q_hf)'$ is understood on the interiors of the grid cells, where it is a polynomial of degree $\le k-1$; all pointwise statements below are for $t$ in the interior of a cell. This convention makes the theory cover continuous piecewise linear interpolation as well as smooth spline quasi-interpolants.

\begin{theorem}[Positive-averaging representation]\label{thm:T4a}
Under the above hypotheses, $(Q_hf)'$ depends on $f$ only through $f'$, and for every $t$ interior to a cell there is a Borel measure $\mu_t$ on $J$ with
\[
(Q_hf)'(t)=\int_J f'\,d\mu_t,\qquad \mu_t\ge0,\qquad \mu_t(J)=1 .
\]
\end{theorem}

\begin{proof}
Reproduction of constants gives $Q_h(f+\mathrm{const})=Q_hf+\mathrm{const}$, hence $f_1'=f_2'$ implies $(Q_hf_1)'=(Q_hf_2)'$. Therefore $T_t(g):=(Q_hf_g)'(t)$, with $f_g(x)=\int_{\inf J}^xg$, is a well-defined linear functional on $C(J)$, bounded by $1$ by \eqref{eq:lipdim}. By the Riesz representation theorem $T_t(g)=\int g\,d\mu_t$ with $\|\mu_t\|_{TV}\le1$. Reproduction of $f(s)=s$ gives $T_t(1)=1$, i.e.\ $\mu_t(J)=1$. From $1=\mu_t(J)\le\|\mu_t\|_{TV}\le1$ we get equality of the mass and the total variation, hence $\mu_t\ge0$.
\end{proof}

No locality assumption is used.

\begin{lemma}[Moments]\label{lem:moments}
Put $m_1(t)=\int(s-t)\,d\mu_t$ and $m_2(t)=\int(s-t)^2\,d\mu_t\ge0$, and let $f_2(x)=x^2/2$, $f_3(x)=x^3/3$. Then
\[
\begin{gathered}
(Q_hf_2)'=f_2'\ \Longrightarrow\ m_1\equiv0,\\
\bigl[(Q_hf_2)'=f_2'\ \text{and}\ (Q_hf_3)'=f_3'\bigr]\ \Longrightarrow\ m_2\equiv0 .
\end{gathered}
\]
Exact reproduction $Q_hf_2=f_2$, $Q_hf_3=f_3$ is of course sufficient for the hypotheses.
\end{lemma}

\begin{proof}
$(Q_hf_2)'(t)=t$ reads $\int s\,d\mu_t=t$, i.e.\ $m_1(t)=0$. Then $(Q_hf_3)'(t)=t^2$ reads $\int s^2d\mu_t=t^2$, whence
$m_2(t)=\int s^2d\mu_t-2t\int s\,d\mu_t+t^2=t^2-2t^2+t^2=0$.
\end{proof}

\begin{lemma}[Exact reproduction from an order hypothesis]\label{lem:repro}
\emph{(a)} If $Q_h$ satisfies the derivative estimate
$\|(Q_hf)'-f'\|_\infty\le Ch^r\|f^{(r+1)}\|_\infty$
with $r\ge2$, then $(Q_hf_2)'=f_2'$; if $r\ge3$, also $(Q_hf_3)'=f_3'$.
\emph{(b)} If $Q_h$ satisfies the $C^0$ estimate
$\|Q_hf-f\|_\infty\le Ch^r\|f^{(r)}\|_\infty$
with $r\ge4$, then $Q_hf=f$ for every polynomial $f$ of degree $\le r-1$, in particular $Q_hf_2=f_2$ and $Q_hf_3=f_3$.
\end{lemma}

\begin{proof}
In (a), $f_2^{(r+1)}=0$ for $r\ge2$ and $f_3^{(r+1)}=0$ for $r\ge3$; in (b), $f^{(r)}=0$ for $\deg f\le r-1$. In both cases the right-hand side of the estimate vanishes.
\end{proof}

Note that Lemma~\ref{lem:moments} is stated in terms of exact reproduction, so that it applies verbatim in both settings of Lemma~\ref{lem:repro}; the derivative estimate is \emph{not} assumed in the $C^0$ case.

\begin{theorem}[Order barrier; rigidity]\label{thm:T4b}
Let $Q_h:C^1(J)\to\Scal_{k,\Delta_h}$ be linear, Lip-diminishing, reproduce affine functions, and satisfy either
\begin{enumerate}[label=\textup{(\roman*)},nosep]
\item the derivative estimate of Lemma~\ref{lem:repro}(a) with $r\ge3$, or
\item the $C^0$ estimate of Lemma~\ref{lem:repro}(b) with $r\ge4$.
\end{enumerate}
Then $Q_hf-f$ is constant for every $f\in C^1(J)$, and consequently $\Scal_{k,\Delta_h}=C^1(J)$, which is false. Hence:
\begin{center}
\emph{a linear Lip-diminishing operator with values in $\Scal_{k,\Delta_h}$ has derivative order at most $2$ and $C^0$ order at most $3$.}
\end{center}
\end{theorem}

\begin{proof}
In either case, Lemma~\ref{lem:repro} shows that $Q_h$ reproduces $f_2$ and $f_3$ in the sense required by Lemma~\ref{lem:moments}: exactly under (ii), and up to an additive constant under (i), which affects neither the derivative nor the moments. By Lemma~\ref{lem:moments}, $m_1\equiv m_2\equiv0$ on cell interiors. By Theorem~\ref{thm:T4a}, $\mu_t$ is a probability measure, and $m_2(t)=0$ forces $\supp\mu_t=\{t\}$, i.e.\ $\mu_t=\delta_t$. Hence $(Q_hf)'(t)=f'(t)$ at every interior point of every cell, so $Q_hf-f$ is constant on each cell, and being continuous on $J$ it is constant on $J$. Since $\Scal_{k,\Delta_h}$ contains the constants, this gives
\[
f=Q_hf-\mathrm{const}_f\in\Scal_{k,\Delta_h}\qquad\text{for every }f\in C^1(J),
\]
which is absurd: $\Scal_{k,\Delta_h}$ is finite-dimensional and $C^1(J)$ is not.
\end{proof}

\begin{remark}
(i) No anchoring is needed: the argument ``$f+\mathrm{const}\in\Scal\Rightarrow f\in\Scal$'' closes the case. (ii) Positivity is used essentially: without \eqref{eq:lipdim}, $m_1=m_2=0$ does not force $\mu_t=\delta_t$. (iii) The statement is pointwise in $t$ and holds for each single $h$, not merely asymptotically. (iv) Continuity of the elements of $\Scal_{k,\Delta_h}$ is used exactly once, to pass from ``constant on each cell'' to ``constant''.
\end{remark}

\begin{definition}[Reference kink profile]\label{def:kink}
Let $G_*(x):=\tfrac12x^2-x_+^2$ and
\[
c_k:=\dist_{L^\infty[-1/2,1/2]}\bigl(G_*,\Pcal_{\max(k-1,1)}\bigr).
\]
Then $c_k>0$ for every $k\ge1$, because $G_*''$ jumps at $0$ and hence $G_*$ is not a polynomial.
\end{definition}

\begin{theorem}[Quantitative range saturation]\label{thm:T4d}
Let $T$ be any map with $(Tf)'\in\Scal_{k-1,\Delta_h}$ cellwise --- in particular $T=Q_h$ as above. Then, with $c_k$ as in Definition~\ref{def:kink},
\[
\sup\bigl\{\|(Tf)'-f'\|_{\infty}:\ f\in W^{3,\infty}(J),\ \|f'''\|_{L^\infty}\le1\bigr\}\ \ge\ c_kh^2 ,
\]
and in fact $\|(Tf)'-f'\|_{L^\infty(I)}\ge c_kh^2$ on \emph{every} cell $I$, for the single function $f$ constructed in the proof.
\end{theorem}

\begin{proof}
Let $g\in C^{1,1}(J)$ satisfy $g''=+1$ on the left half of each cell and $g''=-1$ on the right half, and let $f$ be a primitive of $g$; then $f\in W^{3,\infty}$, $\|f'''\|_\infty=1$ and $f'=g$. Fix a cell $I$ with midpoint $m$. Integrating $g''$ twice,
\[
g(x)=g(m)+g'(m)(x-m)+\tfrac12(x-m)^2-(x-m)_+^2
=h^2G_*\!\Bigl(\frac{x-m}{h}\Bigr)+a(x),
\]
where $a$ is affine; note the coefficient of the kink is $1$, not $2$. Now $(Tf)'|_I\in\Pcal_{k-1}$. If $k\ge2$ then $\Pcal_{k-1}\supseteq\Pcal_1\ni a$, so the affine part is absorbed and, rescaling $I$ to $[-1/2,1/2]$,
\[
\|(Tf)'-g\|_{L^\infty(I)}\ \ge\ \dist_{L^\infty(I)}\bigl(g,\Pcal_{k-1}\bigr)=h^2\dist_{L^\infty[-1/2,1/2]}\bigl(G_*,\Pcal_{k-1}\bigr)=c_kh^2 .
\]
If $k=1$ then $\Pcal_{k-1}=\Pcal_0\subset\Pcal_1$, so $\dist(g,\Pcal_0)\ge\dist(g,\Pcal_1)=h^2\dist(G_*,\Pcal_1)=c_1h^2$. In both cases the affine part causes no loss, which is exactly why $\Pcal_{\max(k-1,1)}$ appears in Definition~\ref{def:kink}.
\end{proof}

\begin{remark}
Theorem~\ref{thm:T4d} uses only the range restriction; neither positivity nor \eqref{eq:lipdim} enters. It should therefore be presented as a spline-range saturation lemma, complementary to the rigidity Theorem~\ref{thm:T4b} rather than part of it.
\end{remark}

\subsection{\texorpdfstring{The $C^0$ barrier is two}{The C0 barrier is two}}

Theorem~\ref{thm:T4b} leaves a gap in the $C^0$ scale: the barrier it produces is $3$, while the largest order realised by a Lip-diminishing operator we know is $2$. We close the gap: the barrier is also $2$. The mechanism is that an order-three $C^0$ estimate forces the averaging measure to be unbiased, and an unbiased positive averaging with spline range must have a variance that is bounded below \emph{on every cell}, which then accumulates over the $|J|/h$ cells.

Throughout, $R_hg:=(Q_hf_g)'$ with $f_g$ a primitive of $g$, so that $R_hg(t)=\int g\,d\mu_t$ and $R_hg|_I\in\Pcal_{k-1}$ on every cell $I$.

\begin{lemma}[The moments are piecewise polynomial]\label{lem:m2poly}
Let $e_1(s)=s$ and $e_2(s)=s^2$. Then
\[
m_1(t)=R_he_1(t)-t,
\qquad
m_2(t)=R_he_2(t)-2t\,R_he_1(t)+t^2 ,
\]
so $m_1$ restricted to a cell is a polynomial of degree at most $\max(k-1,1)$ and $m_2$ of degree at most $d:=\max(k,2)$. If $m_1\equiv0$ then $m_2=R_he_2-t^2$ on each cell, of degree at most $\max(k-1,2)$.
\end{lemma}

\begin{proof}
Expand $\int(s-t)^jd\mu_t$ using $\mu_t(J)=1$; on a cell $R_he_1,R_he_2\in\Pcal_{k-1}$, and the factor $t$ raises the degree by one.
\end{proof}

\begin{lemma}[Signed integrals of a polynomial]\label{lem:signed}
For every $d$ there is $\beta_d>0$ such that for every $q\in\Pcal_d$ and every interval $I$ of length $h$ there is a subinterval $I'\subseteq I$ with
$\bigl|\int_{I'}q\bigr|\ge\beta_d\,h\,\|q\|_{L^\infty(I)}$.
\end{lemma}

\begin{proof}
Let $Q$ be a primitive of $q$, so $Q\in\Pcal_{d+1}$ and $\|Q'\|_{L^\infty(I)}=\|q\|_{L^\infty(I)}$. If $|\int_{I'}q|\le\epsilon$ for every subinterval, then $\operatorname{osc}_I Q\le\epsilon$, hence $\|Q-\bar Q\|_{L^\infty(I)}\le\epsilon$ for a suitable constant $\bar Q$. By Markov's inequality on $\Pcal_{d+1}$, $\|Q'\|_{L^\infty(I)}\le 2(d+1)^2h^{-1}\|Q-\bar Q\|_{L^\infty(I)}\le2(d+1)^2\epsilon/h$. Thus $\epsilon\ge h\|q\|_{L^\infty(I)}/(2(d+1)^2)$.
\end{proof}

\begin{lemma}[Norm equivalence on a cell]\label{lem:normeq}
For every $d$ there is $\gamma_d>0$ such that $\|q\|_{L^1(I)}\ge\gamma_d\,h\,\|q\|_{L^\infty(I)}$ for every $q\in\Pcal_d$ and every interval $I$ of length $h$.
\end{lemma}

\begin{proof}
On the reference interval $[0,1]$ the two norms are norms on the finite-dimensional space $\Pcal_d$, hence equivalent; the affine change of variables scales both sides by $h$.
\end{proof}

\begin{lemma}[Cellwise lower bound on the variance]\label{lem:m2lower}
Let $\lambda:=\|m_1\|_{L^\infty(J)}$. Then on every cell $I$,
\[
\sup_{t\in I}m_2(t)\ \ge\ 2c_kh^2-\tfrac h2\lambda ,
\]
with $c_k$ as in Theorem~\ref{thm:T4d}. In particular $\sup_Im_2\ge2c_kh^2$ when $m_1\equiv0$.
\end{lemma}

\begin{proof}
Let $g$ and $f$ be as in the proof of Theorem~\ref{thm:T4d}, so $\|f'''\|_\infty=1$, $f'=g$ and $\Lip(g')\le1$. We avoid a Lagrange remainder, which is not available for $g\in C^{1,1}$ with piecewise defined $g''$. Instead,
\[
\begin{gathered}
g(s)-g(t)-g'(t)(s-t)=\int_t^s\bigl(g'(u)-g'(t)\bigr)du,\\
\text{so}\qquad
\bigl|g(s)-g(t)-g'(t)(s-t)\bigr|\le\tfrac12|s-t|^2 .
\end{gathered}
\]
Integrating against $\mu_t$ and using $\mu_t\ge0$, $\mu_t(J)=1$,
\[
|R_hg(t)-g(t)|\le|g'(t)|\,|m_1(t)|+\tfrac12m_2(t)\le\tfrac h4\lambda+\tfrac12m_2(t),
\]
where we used that $g'$ is the periodic zigzag obtained by integrating $g''=\pm1$ over half-cells, so that, after fixing the free constant, $\|g'\|_{L^\infty(J)}\le h/4$. (This normalisation is also what makes the test function budget-compatible: $\|f'\|_\infty=\|g\|_\infty\le h|J|/8$.) On the other hand $R_hg|_I\in\Pcal_{k-1}$, so by the proof of Theorem~\ref{thm:T4d}, $\|R_hg-g\|_{L^\infty(I)}\ge c_kh^2$. Combining the two gives the claim.
\end{proof}

\begin{theorem}[$C^0$ saturation at order two, modulo constants]\label{thm:T4e}
Let $J$ be partitioned into $N=|J|/h$ cells, and let $Q_h:C^1(J)\to\Scal_{k,\Delta_h}$ be linear, Lip-diminishing, and reproduce affine functions. Then, with a constant $\Theta_k>0$ depending only on $k$, the following holds \emph{unconditionally}:
\[
\sup\bigl\{\operatorname{osc}_J(Q_hf-f):\ f\in W^{3,\infty}(J),\ \|f'''\|_{L^\infty}\le1\bigr\}\ \ge\ \Theta_k\,|J|\,h^2 .
\]
Consequently, for the anchored operator $Q^{t_0}_hf:=Q_hf-(Q_hf-f)(t_0)$ and any $t_0\in J$,
\[
\sup\bigl\{\|Q^{t_0}_hf-f\|_{L^\infty(J)}:\ \|f'''\|_{L^\infty}\le1\bigr\}\ \ge\ \tfrac12\Theta_k|J|h^2 ,
\]
and no $C^0$ estimate of order $r\ge3$ can hold, for $h$ below an explicit threshold.
\end{theorem}

\begin{proof}
The dichotomy is on the \emph{derivative} of the quadratic defect $E_2:=Q_hf_2-f_2$, $f_2(x)=x^2/2$; note $E_2'=R_he_1-e_1=m_1$ by Lemma~\ref{lem:m2poly}, and $f_2'''=0$.

\emph{Case A: $m_1\not\equiv0$.} Then $E_2$ is not constant, so $\operatorname{osc}_JE_2>0$, and for every $A>0$ the function $Af_2$ lies in the ball and satisfies $\operatorname{osc}_J(Q_h(Af_2)-Af_2)=A\operatorname{osc}_JE_2$. The supremum is $+\infty$.

\emph{Case B: $m_1\equiv0$.} Lemmas~\ref{lem:m2poly}, \ref{lem:normeq} and \ref{lem:m2lower} give, on every cell $I$,
\[
\int_Im_2\ \ge\ \gamma_d\,h\,\sup_Im_2\ \ge\ 2\gamma_dc_k\,h^3 .
\]
Take $f_3(x)=x^3/6$, so $\|f_3'''\|_\infty=1$ and $g:=f_3'=x^2/2$ has $g''\equiv1$; for a quadratic the expansion is exact, so
\[
(Q_hf_3)'(t)-f_3'(t)=R_hg(t)-g(t)=\tfrac12m_2(t)\ \ge\ 0
\]
at every interior point of every cell. Hence $Q_hf_3-f_3$ is Lipschitz with a nonnegative derivative a.e., so nondecreasing, and
\[
\operatorname{osc}_J(Q_hf_3-f_3)=\int_J\tfrac12m_2\ \ge\ N\gamma_dc_kh^3=\gamma_dc_k|J|h^2=:\Theta_k|J|h^2 .
\]

The anchored form follows from $\|E-E(t_0)\|_{L^\infty(J)}\ge\tfrac12\operatorname{osc}_JE$, and the order statement because an estimate $\operatorname{osc}(Q_hf-f)\le2Ch^3\|f'''\|_\infty$ would contradict the display once $h<\Theta_k|J|/(2C)$.
\end{proof}

\begin{remark}[Why oscillation and not the sup-norm]\label{rem:osc}
The earlier version of this theorem was stated for $\|Q_hf-f\|_\infty$ with the dichotomy cut at $Q_hf_2\ne f_2$. That statement is true but not the right one for our application: an operator with $Q_hf_2=f_2+C$, $C\ne0$, satisfies $(Q_hf_2)'=f_2'$ and is harmless after anchoring --- indeed the anchoring of Section~\ref{sec:setting} removes $C$ for free --- yet the old Case A declared it infinitely bad. Quotienting by constants, i.e.\ working with $\operatorname{osc}$ or equivalently with the anchored error, makes the statement invariant under that harmless gauge, and the variance argument of Case B in fact bounds the oscillation directly, so nothing is lost.
\end{remark}

\begin{theorem}[Budget-compatible class-wide saturation]\label{thm:budgetsat}
Let $Q_h$ be as above and let $B,K>0$ be arbitrary. Put
\[
\mathfrak E_h(Q_h;B,K):=\sup\Bigl\{\operatorname{osc}_J(Q_hf-f):\ \|f'\|_{L^\infty}\le B,\ \|f'''\|_{L^\infty}\le K\Bigr\},
\]
where
\[
K_{\mathrm{eff}}:=\min\Bigl\{K,\ \frac{8B}{|J|^2}\Bigr\}.
\]
Then, with constants depending only on $k$,
\[
\begin{gathered}
\mathfrak E_h(Q_h;B,K)\ \ge\ c_*\,\min\Bigl\{K_{\mathrm{eff}}|J|h^2,\ \frac{B\,h^3}{|J|^{2}}\Bigr\},\\
c_*:=\min\Bigl\{\tfrac14\gamma_dc_k,\ \tfrac12\beta_d\gamma_dc_k\Bigr\}.
\end{gathered}
\]
No smallness assumption on $h$ and no restriction on $(B,K)$ is needed. Since the bound holds for \emph{every} operator of the class with the same constant, it is simultaneously a lower bound for $\inf_{Q_h}\mathfrak E_h(Q_h;B,K)$, i.e.\ a minimax statement; Open problem~\ref{op:minimax} asks whether the $\min$ can be removed from it.
\end{theorem}

\begin{proof}
Put $\lambda:=\|m_1\|_{L^\infty(J)}$ and $\lambda_0:=\gamma_dc_kh^2/(2|J|)$. The inequality $\lambda_0\le2c_kh$, needed below, reduces after cancelling $c_kh>0$ to $\gamma_dh\le4|J|$, and this holds automatically: testing Lemma~\ref{lem:normeq} on $q\equiv1$ gives $\gamma_d\le1$, while the grid has at least one cell, so $h\le|J|$. (The constant $c_k$ cancels here; an earlier draft carried it along, which made the condition look like a genuine restriction.)

\emph{Case $\lambda\le\lambda_0$.} By Lemma~\ref{lem:m2lower}, $\sup_Im_2\ge2c_kh^2-\tfrac h2\lambda\ge c_kh^2$ on every cell, so $\int_Jm_2\ge\gamma_dc_k|J|h^2$ by Lemma~\ref{lem:normeq}. Test with $f=K_{\mathrm{eff}}(t-t_c)^3/6$, $t_c$ the midpoint of $J$: then $\|f'''\|_\infty=K_{\mathrm{eff}}\le K$ and $\|f'\|_\infty=K_{\mathrm{eff}}|J|^2/8\le B$ by the definition of $K_{\mathrm{eff}}$, so $f$ lies in the ball for \emph{all} $B,K>0$. Writing $g=f'/K_{\mathrm{eff}}=(t-t_c)^2/2$, exactness of the expansion for a quadratic gives $(Q_hf)'-f'=K_{\mathrm{eff}}\bigl[(t-t_c)m_1(t)+\tfrac12m_2(t)\bigr]$, whence
\[
\begin{aligned}
\operatorname{osc}_J(Q_hf-f)&\ \ge\ K_{\mathrm{eff}}\Bigl|\int_J\bigl[(t-t_c)m_1+\tfrac12m_2\bigr]\Bigr|\\
&\ \ge\ K_{\mathrm{eff}}\Bigl(\tfrac12\gamma_dc_k|J|h^2-\tfrac{|J|^2}{2}\lambda\Bigr)
\ \ge\ \tfrac14\gamma_dc_kK_{\mathrm{eff}}|J|h^2 ,
\end{aligned}
\]
using $\lambda\le\lambda_0$ in the last step.

\emph{Case $\lambda>\lambda_0$.} Pick a cell $I$ with $\|m_1\|_{L^\infty(I)}\ge\lambda/2$. By Lemma~\ref{lem:signed} there is a subinterval on which $|\int m_1|\ge\beta_dh\lambda/2$, and since $E_2'=m_1$ this means $\operatorname{osc}_JE_2\ge\beta_dh\lambda/2$. Test with $f=Af_2^{(c)}$, $f_2^{(c)}(t)=(t-t_c)^2/2$: then $\|f'''\|_\infty=0\le K$ and $\|f'\|_\infty=A|J|/2$, so $A=2B/|J|$ is admissible, and
\[
\operatorname{osc}_J(Q_hf-f)=A\operatorname{osc}_JE_2\ \ge\ \frac{2B}{|J|}\cdot\frac{\beta_dh\lambda_0}{2}=\frac{\beta_d\gamma_dc_k}{2}\cdot\frac{B h^3}{|J|^2}.
\]
Taking the smaller of the two bounds gives the statement.
\end{proof}

\begin{remark}[What this does and does not give]\label{rem:budgetsat}
Theorem~\ref{thm:budgetsat} is class-wide and budget-compatible: no hypothesis on $Q_h$ beyond the class, and both test functions respect a prescribed bound on $\|f'\|_\infty$, which is what a layer budget controls. The two branches are not symmetric. The branch $m_1\equiv0$ contains, in particular, every operator of derivative order at least two, so it covers Schoenberg's operator and the extremiser of Proposition~\ref{ex:midpoint}; there the rate is the expected $K_{\mathrm{eff}}|J|h^2$. The branch $m_1\not\equiv0$ rules out derivative order two or higher, as for piecewise linear interpolation; there the guaranteed rate degrades by a factor $h/|J|=1/G$. Note that the implication runs one way only: derivative order at least two forces $m_1\equiv0$, but $m_1\equiv0$ does not by itself give an order estimate on all of $C^1$. We do not know whether that loss is an artefact of the proof; closing it is Open problem~\ref{op:minimax}. The scaling escape of Case A in Theorem~\ref{thm:T4e} is, finally, \emph{not} available here: it uses $\|f'\|_\infty\sim A\to\infty$ and therefore leaves the Lipschitz ball, which is precisely the reason Theorem~\ref{thm:T4e} alone does not yield a budget-compatible statement.
\end{remark}

\begin{remark}
Theorem~\ref{thm:T4e} supersedes case (ii) of Theorem~\ref{thm:T4b} and closes what was an open problem in an earlier draft. Note that no order hypothesis is assumed: the dichotomy on $m_1=E_2'$ removes it, Case A being the degenerate branch in which the operator fails to reproduce quadratics modulo constants and the supremum is infinite. Two further features are worth isolating. First, it is not asymptotic: the conclusion holds for each fixed $h$ below an explicit threshold. Second, the accumulation over cells is essential --- the per-cell bound alone gives $h^3$, which is compatible with order three, and only summing over the $|J|/h$ cells produces $|J|h^2$. This is the spline-range, non-asymptotic counterpart of the classical saturation of positive linear approximation processes at the second order; we make no claim of novelty for the phenomenon itself, only for this form of the statement.
\end{remark}

\begin{proposition}[Both barriers are attained on a bounded interval]\label{ex:midpoint}
Let $J=[a,b]$, $x_i=a+ih$, $i=0,\dots,N$. Define probability measures on $J$ by
\[
\nu_i:=\tfrac1h\mathbf 1_{[x_i-h/2,\,x_i+h/2]}(s)\,ds\quad(1\le i\le N-1),
\qquad \nu_0:=\delta_a,\quad \nu_N:=\delta_b,
\]
so that $\nu_i(J)=1$ and $\int s\,d\nu_i(s)=x_i$ for every $i$. For $t=(1-\theta)x_i+\theta x_{i+1}$ with $\theta\in[0,1]$ put $\mu_t:=(1-\theta)\nu_i+\theta\nu_{i+1}$ and
\[
R_hg(t):=\int_Jg\,d\mu_t,
\qquad
Q_hf(t):=f(a)+\int_a^tR_h(f')(s)\,ds .
\]
Then $Q_h$ is linear, maps $C^1(J)$ into the continuous piecewise quadratics on the grid $\Delta_h$, is Lip-diminishing, reproduces $\Pcal_2$, and it satisfies
\[
\|(Q_hf)'-f'\|_{L^\infty(J)}\le Ch^2\|f'''\|_\infty,
\qquad
\|Q_hf-f\|_{L^\infty(J)}\le C|J|h^2\|f'''\|_\infty .
\]
Hence both barriers --- derivative order $2$ (Theorem~\ref{thm:T4b}) and $C^0$ order $2$ (Theorem~\ref{thm:T4e}) --- are attained, and the pair is sharp.
\end{proposition}

\begin{proof}
Each $\mu_t$ is a convex combination of probability measures on $J$, hence a probability measure on $J$, and its barycentre is $(1-\theta)x_i+\theta x_{i+1}=t$; thus $m_1\equiv0$. Positivity and unit mass give $\|R_hg\|_\infty\le\|g\|_\infty$, i.e.\ $Q_h$ is Lip-diminishing. On each cell $t\mapsto R_hg(t)$ is affine, so $R_h(f')$ is continuous piecewise linear and $Q_hf$ is a continuous piecewise quadratic. If $f\in\Pcal_2$ then $g=f'$ is affine and $R_hg(t)=g(t)$ by the barycentre property, so $Q_hf=f$ by the anchoring at $a$. Finally $\supp\mu_t\subset[t-\tfrac32h,\,t+\tfrac32h]\cap J$, so $m_2(t)\le\tfrac94h^2$, and by the estimate of Lemma~\ref{lem:m2lower} applied in the direction of an upper bound,
$|R_hg(t)-g(t)|\le\tfrac12m_2(t)\Lip(g')\le\tfrac98h^2\|f'''\|_\infty$;
integrating from $a$ to $t$ gives the $C^0$ bound.
\end{proof}

\begin{center}
\begin{tabular}{lcc}
\hline
& derivative order & $C^0$ order\\
\hline
barrier for the class & $\le2$ & $\le2$\\
midpoint-average operator (Prop.~\ref{ex:midpoint}) & $2$ & $2$\\
Schoenberg operator & $2$ & $2$\\
piecewise linear interpolation ($k=1$) & $1$ & $2$\\
\hline
\end{tabular}
\end{center}

For piecewise linear interpolation the derivative on a cell is the secant slope, so $\|(I_hf)'-f'\|_\infty=O(h)\|f''\|_\infty$ and the derivative order is $1$, while the $C^0$ order is $2$; the values are consistent with, and far from saturating, the barriers.

\begin{remark}[Position with respect to classical approximation theory]\label{rem:korovkin}
The implication ``$\mu_t\ge0$, $\int1\,d\mu_t=1$, $\int s\,d\mu_t=t$, $\int s^2d\mu_t=t^2\Rightarrow\mu_t=\delta_t$'' is the pointwise mechanism behind Korovkin's theorem, whose test set is exactly $\{1,x,x^2\}$ \cite{Korovkin}; second-order saturation of positive linear approximation processes is classical. Our contribution is not that mechanism but its combination with two structural hypotheses specific to budget-preserving discretisation: (a) that the Lipschitz contraction is required at the level of the \emph{derivative}, which is what produces the positive averaging representation; and (b) that the range is a spline space, which converts $\mu_t=\delta_t$ into a contradiction. The neighbourhood is in fact denser than the Korovkin mechanism alone, and we state it explicitly. General frameworks for \emph{lower} estimates of linear operators whose range is smooth are available: Nagler \cite{Nagler} derives converse estimates from the fixed-point space and the smoothness of the range, in terms of moduli of smoothness and $K$-functionals, and applies them to positive finite-rank operators and to Schoenberg's operator. Quantitative lower bounds for the variation-diminishing Schoenberg operator, in terms of classical moduli, are due to Nagler, Cerejeiras and Forster \cite{NCF}; lower estimates for its second moment and simultaneous estimates for its derivatives, including the quadratic order in the mesh size, are also known \cite{BGKT,Tachev,GWZ}. Saturation of positive operators at the second order goes back to Lorentz and Schumaker \cite{LorentzSchumaker}, with the wider theory surveyed in \cite{DeVore}.

Accordingly we do not present Theorems~\ref{thm:T4d} and \ref{thm:T4e} as novelties. They are self-contained fixed-$h$ estimates adapted to the operator class at hand: the first is a range-saturation lemma, the second a saturation estimate modulo constants, and both exist here to feed Theorem~\ref{thm:budgetsat}. What distinguishes Theorem~\ref{thm:budgetsat} from all of the above is not the mechanism but the problem: the operator is assumed only to be derivative-contracting and spline-valued, while the adversarial function is constrained \emph{simultaneously} in the Lipschitz seminorm that the stability budget controls and in a third-derivative smoothness seminorm. To the best of our knowledge that minimax problem has not been considered.
\end{remark}

\section{Sharpness under depth composition}
\label{sec:sharp}

We show that the linear-in-$L$ accumulation of Theorem~\ref{thm:Seta} is attained, i.e.\ that layer errors need not cancel, and that without a budget the amplification is genuinely exponential. The construction is one-dimensional and completely explicit. Anchoring is not imposed here (Remark~\ref{rem:anchor}(iii)); the trajectories are localised directly.

\subsection{The second moment of the Schoenberg weights}

Let $\Delta_h$ be a uniform grid of spacing $h$ on an interval $J$, let $N_{i,k}$ be the B-splines of degree $k$ and $\xi_i$ the Greville abscissae, and let $V=V_{h,k}$ be the Schoenberg operator $Vf=\sum_if(\xi_i)N_{i,k}$. Put
\[
M_2(t):=\sum_i N_{i,k}(t)(\xi_i-t)^2 .
\]

\begin{lemma}[The second moment is constant]\label{lem:secondmoment}
Let $k\ge2$, let the grid be uniform with spacing $h$, and let $J_h\subset J$ be the set of points at distance at least $(k+1)h$ from $\partial J$. Then
\[
M_2(t)\ =\ \frac{k+1}{12}\,h^2\qquad\text{for every }t\in J_h ,
\]
independently of $t$. For $k=1$ the statement fails: there $M_2(t)=(x_{j+1}-t)(t-x_j)$ on the cell $[x_j,x_{j+1}]$, which vanishes at the knots.
\end{lemma}

\begin{proof}
Write $a_j:=t_{i+j}$, $j=1,\dots,k$, for the knot window of $N_{i,k}$, so that $\xi_i=\frac1k\sum_ja_j$. Marsden's identity $(y-t)^k=\sum_i\psi_{i,k}(y)N_{i,k}(t)$ with $\psi_{i,k}(y)=\prod_{j=1}^k(y-t_{i+j})$ gives, on comparing the coefficients of $y^k$, $y^{k-1}$ and $y^{k-2}$,
\[
\sum_iN_{i,k}(t)=1,
\qquad
\sum_i\xi_iN_{i,k}(t)=t,
\qquad
\sum_i\eta_iN_{i,k}(t)=t^2 ,
\]
where $\eta_i:=\frac{2}{k(k-1)}\sum_{1\le r<s\le k}t_{i+r}t_{i+s}$ is the normalised second elementary symmetric function of the window; the third identity requires $k\ge2$. Consequently
\[
M_2(t)=\sum_iN_{i,k}(t)\xi_i^2-t^2=\sum_iN_{i,k}(t)\bigl(\xi_i^2-\eta_i\bigr).
\]
Now, with $\Sigma_1=\sum_ja_j$ and $\Sigma_2=\sum_ja_j^2$,
\[
\xi_i^2-\eta_i=\frac{\Sigma_1^2}{k^2}-\frac{\Sigma_1^2-\Sigma_2}{k(k-1)}
=\frac{k\Sigma_2-\Sigma_1^2}{k^2(k-1)}
=\frac{1}{k-1}\Bigl(\frac{\Sigma_2}{k}-\xi_i^2\Bigr)
=\frac{\operatorname{Var}(a)}{k-1},
\]
the population variance of the $k$ knots in the window. For a uniform grid the window is an arithmetic progression of $k$ terms with step $h$, so $\operatorname{Var}(a)=h^2(k^2-1)/12$, independently of $i$, and
\[
\xi_i^2-\eta_i=\frac{h^2(k^2-1)}{12(k-1)}=\frac{k+1}{12}h^2 .
\]
Substituting and using $\sum_iN_{i,k}=1$ gives the claim. For $k=1$ the Greville abscissae are the knots, $N_{i,1}(\xi_j)=\delta_{ij}$, and the displayed formula for $M_2$ follows from the two hat weights.
\end{proof}

\begin{remark}
For the Schoenberg operator the quantity $M_2$ computed here is exactly the second moment $m_2$ of the representing measures of Section~\ref{sec:rigidity}, so Lemma~\ref{lem:secondmoment} also identifies $m_2$ for that operator: $m_2\equiv(k+1)h^2/12$. In particular the general cellwise bound $\sup_Im_2\ge2c_kh^2$ of Lemma~\ref{lem:m2lower} is far from sharp for Schoenberg's operator, as it must be, since $c_k$ is a constant valid for the whole class.

The identity is exact and $t$-independent, which is what removes the last unspecified constant from this section: below, the lower and upper per-step contributions coincide, and the constant in Corollary~\ref{cor:sharpbudget} becomes $e^{-2c}$ with no residual ratio. The restriction $k\ge2$ is not technical: it is exactly the case in which the third Marsden identity is available, and the $k=1$ formula shows the statement is false without it.
\end{remark}

\subsection{The tower}

Fix $k\ge2$, $\rho>0$, a grid of spacing $h$ on $J=[0,\rho+(k+1)h]$, and let $J_h$ be as above, so $[(k+1)h,\rho]\subset J_h$. Fix $\gamma\in(0,1)$ and $\mu>0$ with
\begin{equation}\label{eq:gammamu}
\gamma+\mu\,|J|\le 1,
\end{equation}
and set
\[
g(t)=\gamma t+\tfrac{\mu}{2}t^2,\qquad \varphi_\ell:=qg,\qquad \widehat\varphi_\ell:=qVg\quad(\ell=0,\dots,L-1),
\]
with $q>0$. Then $g(0)=0$, $g'=\gamma+\mu t\in[\gamma,1]$ on $J$, $g''\equiv\mu$, and
\[
\lambda_\ell=q\|g'\|_{L^\infty(J)}\le q,\qquad m:=q\gamma\le\inf_J q(Vg)' .
\]

\begin{lemma}\label{lem:exactquad}
For the quadratic $g$ above, $Vg(t)-g(t)=\tfrac{\mu}{2}M_2(t)$ for all $t$, and $(Vg)'\ge\gamma$ on $J$.
\end{lemma}

\begin{proof}
By Marsden's identity,
$Vg(t)=\sum_ig(\xi_i)N_i(t)=g(t)+g'(t)\sum_i(\xi_i-t)N_i(t)+\tfrac{\mu}{2}\sum_i(\xi_i-t)^2N_i(t)$,
and the middle sum vanishes. The derivative bound is the variation-diminishing property recalled in Section~\ref{sec:regimes}: $(Vg)'$ is a convex combination of divided differences of $g$, each of which is at least $\inf_Jg'=\gamma$.
\end{proof}

Write $x_{\ell+1}=\varphi(x_\ell)$, $\hat x_{\ell+1}=\widehat\varphi(\hat x_\ell)$, $x_0=\hat x_0\in(0,\rho]$, and $e_\ell=\hat x_\ell-x_\ell$. Put
\[
\beta:=\frac{q\mu}{2}\cdot\frac{k+1}{12}h^2 ,
\]
by Lemma~\ref{lem:secondmoment}. Since $M_2$ is constant on $J_h$, the upper and lower per-step contributions coincide: with $B:=\beta$, both halves of the recurrence below carry the same constant. Consistently with Theorem~\ref{thm:Seta}, where the layer error is only ever evaluated on the region actually visited, we measure the layer errors of this section on $J_h$:
\[
\varepsilon_\ell:=\|\varphi_\ell-\widehat\varphi_\ell\|_{L^\infty(J_h)}=\beta ,
\]
so that $\sum_{\ell<L}\varepsilon_\ell=L\beta$ exactly. On all of $J$ one would have to absorb the boundary weights of the Schoenberg operator into the constant, which we avoid.

\begin{proposition}[No cancellation, with a two-sided recurrence]\label{prop:sharp}
Assume \eqref{eq:gammamu} and $x_0\ge(k+1)h/\min(1,m)^{L}$. Suppose $\hat x_\ell\le\rho$ for all $\ell<L$. Then $e_\ell\ge0$ for all $\ell$ and
\begin{equation}\label{eq:accum}
m\,e_\ell+\beta\ \le\ e_{\ell+1}\ \le\ q\,e_\ell+B ,
\end{equation}
hence
\[
\beta\cdot\frac{m^L-1}{m-1}\ \le\ e_L\ \le\ B\cdot\frac{q^L-1}{q-1}
\]
(with the usual reading $\beta L$, resp.\ $BL$, when $m=1$, resp.\ $q=1$).
\end{proposition}

\begin{proof}
By $g'>0$ and Lemma~\ref{lem:exactquad}, both $\varphi$ and $\widehat\varphi$ are increasing. We prove $e_\ell\ge0$ and $x_\ell,\hat x_\ell\in J_h$ by induction. For $\ell=0$ this holds by the assumption on $x_0$. Assume it at $\ell$. Since $g(t)\le t$ on $J$ by \eqref{eq:gammamu} and $g(t)\ge\gamma t$, we get $m^\ell x_0\le x_\ell$, so $x_\ell\ge(k+1)h$, and $x_\ell\le\hat x_\ell\le\rho$ by hypothesis; hence both lie in $J_h$. Now
\[
e_{\ell+1}=q\bigl[Vg(\hat x_\ell)-Vg(x_\ell)\bigr]+q\bigl[Vg-g\bigr](x_\ell)
\ \ge\ q\gamma\,e_\ell+\tfrac{q\mu}{2}M_2(x_\ell)\ \ge\ m e_\ell+\beta,
\]
using monotonicity of $Vg$ with $(Vg)'\ge\gamma$ for the first bracket and Lemmas~\ref{lem:exactquad}, \ref{lem:secondmoment} for the second, the latter now with equality; both terms are nonnegative, so $e_{\ell+1}\ge0$. For the upper bound, $V$ is Lip-diminishing and $\|g'\|_{L^\infty(J)}\le1$ by \eqref{eq:gammamu}, so $|Vg(\hat x_\ell)-Vg(x_\ell)|\le e_\ell$, while $[Vg-g](x_\ell)=\tfrac\mu2M_2(x_\ell)=\tfrac\mu2\cdot\tfrac{k+1}{12}h^2$ by Lemmas~\ref{lem:exactquad} and \ref{lem:secondmoment}; this gives $e_{\ell+1}\le qe_\ell+B$ with $B=\beta$. Iterating \eqref{eq:accum} from $e_0=0$ gives both geometric sums.
\end{proof}

\begin{corollary}[The constant in Theorem~\ref{thm:Seta} cannot be improved to $o(1)$]\label{cor:sharpbudget}
Let $c>0$ and $L\ge2c$. Choose
\[
q=1+\frac cL,\qquad \gamma=\frac{1-c/L}{1+c/L}\quad(\text{so }m=q\gamma=1-\tfrac cL),
\qquad \mu=\frac{1-\gamma}{2|J|},
\]
which satisfies \eqref{eq:gammamu}, and $x_0=(k+1)h/m^L$. Then for every $h$ small enough that
\[
e^{3c}(k+1)h\le\frac\rho4
\qquad\text{and}\qquad
e^{c}LB\le\frac\rho4 ,
\]
the standing hypothesis of Proposition~\ref{prop:sharp} holds, so the construction is unconditional, and
\[
e_L\ \ge\ e^{-2c}\beta L=e^{-2c}\sum_{\ell<L}\varepsilon_\ell ,
\]
that is,
\[
\|F-\widehat F\|_{C^0}\ \ge\ e^{-2c}\sum_{\ell<L}\varepsilon_\ell ,
\]
with no residual ratio of constants, because $M_2$ is constant by Lemma~\ref{lem:secondmoment}.
The layer errors do not cancel, and the accumulation is linear in $L$, matching Theorem~\ref{thm:Seta} up to a constant depending only on $k$ and $c$.
\end{corollary}

\begin{proof}
Both bounds of \eqref{eq:accum} are available once the trajectories stay in $J_h$, so we verify this. By \eqref{eq:gammamu}, $x_\ell\le q^\ell x_0\le q^Lx_0$, and since $m=q\gamma$,
\[
q^Lx_0=(k+1)h\Bigl(\frac qm\Bigr)^L=(k+1)h\,\gamma^{-L}\le e^{3c}(k+1)h\le\frac\rho4 ,
\]
where $\gamma^{-L}=\bigl((1+c/L)/(1-c/L)\bigr)^L\le e^{3c}$ because $\log(1+x)\le x$ and $-\log(1-x)\le2x$ for $0\le x=c/L\le\tfrac12$. In particular $x_\ell\le\rho/4$; we do not pass through the cruder chain $x_0\le e^{2c}(k+1)h$, $x_\ell\le e^cx_0$, which would not give the stated constant. The upper half of \eqref{eq:accum} gives
$e_\ell\le B(q^\ell-1)/(q-1)\le Bq^{L-1}L\le e^cLB\le\rho/4$,
so $\hat x_\ell=x_\ell+e_\ell\le\rho/2\le\rho$ and $x_\ell\ge m^\ell x_0\ge(k+1)h$; both trajectories lie in $J_h$. For the lower bound, $\sum_{j<L}m^j\ge Lm^{L-1}\ge L(1-c/L)^{L-1}\ge Le^{-2c}$, using $\log(1-x)\ge-2x$ for $0\le x\le\tfrac12$, which is where $L\ge2c$ is used. Note that $B\to0$ as $h\to0$ with all other parameters fixed, so the two smallness conditions are satisfiable.
\end{proof}

\begin{corollary}[Exponential amplification without a budget]\label{cor:sharpexp}
Let $q>1$ be fixed with $m=q\gamma>1$, and let $x_0=(k+1)h$. If
\[
q^Lx_0\le\frac\rho2
\qquad\text{and}\qquad
B\,\frac{q^L-1}{q-1}\le\frac\rho2 ,
\]
then the hypotheses of Proposition~\ref{prop:sharp} hold and
\[
e_L\ \ge\ \beta\,\frac{m^L-1}{m-1}.
\]
\end{corollary}

\begin{proof}
By \eqref{eq:gammamu}, $x_{\ell+1}\le qx_\ell$, so $x_\ell\le q^Lx_0\le\rho/2$. The \emph{upper} half of \eqref{eq:accum} --- and not the lower one, which controls nothing from above --- gives $e_\ell\le B(q^\ell-1)/(q-1)\le\rho/2$, so $\hat x_\ell=x_\ell+e_\ell\le\rho$, and the standing hypothesis of Proposition~\ref{prop:sharp} is verified inductively. The lower bound is then the lower half of \eqref{eq:accum}. Note that the upper dynamics is governed by $q$ and the lower one by $m=q\gamma$.
\end{proof}

\begin{remark}[Saturation is unavoidable and is stated as such]
The two conditions in Corollary~\ref{cor:sharpexp} express that the accumulated error has not yet filled the domain; no exponential lower bound can hold beyond that point, since both trajectories stay in $J$. Equivalently, an expanding map admits no invariant compact interval, and the construction compensates by starting at the scale $q^{-L}\rho$. This is why the exponential statement is necessarily of the form ``for all $L$ below an explicit saturation threshold''.
\end{remark}

\subsection{Non-cancellation for the whole Lip-diminishing class}
\label{subsec:universal}

Corollaries~\ref{cor:sharpbudget} and \ref{cor:sharpexp} concern one operator, Schoenberg's. We now show that non-cancellation is a property of the whole class of Section~\ref{sec:rigidity}: for \emph{every} linear Lip-diminishing spline-valued operator there is a stable tower on which the layer errors add up rather than cancel. The construction is different from the one above --- it makes the exact trajectory a fixed point --- and it uses Theorem~\ref{thm:T4e} as its only quantitative input.

\begin{theorem}[Universal non-cancellation]\label{thm:universal}
Let $Q_h:C^1(J)\to\Scal_{k,\Delta_h}$ be linear, Lip-diminishing and reproduce affine functions; no further hypothesis is imposed. Let $c>0$, $L\ge2c$ and $\sigma\in(0,1)$. Then there exist $t_*$ in the interior of $J$ and a profile $\omega$ such that, for \emph{every} $\eta>0$ satisfying the two explicit conditions \eqref{eq:etacond} below, the scalar depth-$L$ tower with all edges equal to
\[
\varphi:=\mathrm{id}+\eta\,\omega
\]
satisfies: $\Lip(\varphi)\le1+c/L$ and $\Lip(\widehat\varphi)\le1+c/L$, hence the budget; the exact trajectory issued from $x_0=t_*$ is constant; both trajectories remain in $J$; and
\[
\|F-\widehat F\|_{C^0}\ \ge\ (1-\sigma)\,e^{-2c}\sum_{\ell<L}\varepsilon_\ell,
\qquad
\varepsilon_\ell=\|\varphi-\widehat\varphi\|_{L^\infty(J)} .
\]
Moreover the profile can be chosen as follows, according to a dichotomy on $m_1$:
\begin{enumerate}[label=\textup{(\alph*)},nosep]
\item if $m_1\equiv0$, then $\omega$ is a cubic, $\|\varphi'''\|_\infty=\eta$, and $\varepsilon_\ell\ge\eta\,\Theta_k|J|h^2$ with $\Theta_k$ from Theorem~\ref{thm:T4e};
\item if $m_1\not\equiv0$, then $\omega$ is a quadratic, $\varphi'''\equiv0$, and $\varepsilon_\ell=\eta\,\|E_2\|_{L^\infty(J)}$ with $E_2=Q_hf_2-f_2$, a quantity that is positive but for which we claim no rate in $h$. (The defect is unchanged by subtracting the affine tangent, since $Q_ha=a$; hence $\|D\|_\infty=\|E_2\|_\infty$ and not $\|E_2-E_2(t_*)\|_\infty$.)
\end{enumerate}
\end{theorem}

\begin{proof}
\emph{Step 1: the profile.} Fix a sign $s\in\{-1,+1\}$ to be chosen, and set
\[
\omega_0:=\begin{cases} s\,t^3/6, & \text{if }m_1\equiv0\ \ \text{(case (a))},\\[2pt]
s\,t^2/2, & \text{if }m_1\not\equiv0\ \ \text{(case (b))},\end{cases}
\qquad D:=Q_h\omega_0-\omega_0 .
\]
In case (a), as in Case B of the proof of Theorem~\ref{thm:T4e}, $sD'=\tfrac12m_2\ge0$ on cell interiors, so $sD$ is nondecreasing with $\operatorname{osc}_JD\ge2\Theta_k|J|h^2$, hence $\|D\|_\infty\ge\Theta_k|J|h^2>0$. In case (b), $D=sE_2$ with $E_2'=m_1\not\equiv0$, so $E_2$ is not constant and in particular $\|D\|_\infty>0$. In both cases $D\not\equiv0$; no hypothesis on $Q_h$ beyond membership of the class is used.

We do \emph{not} modify $D$. As $D$ is continuous, choose $t_*$ in the interior of $J$ with
\[
|D(t_*)|\ \ge\ (1-\sigma)\|D\|_{L^\infty(J)}\ >\ 0 ,
\qquad \delta:=\dist(t_*,\partial J)>0 ,
\]
and choose the sign $s$ so that $D(t_*)>0$. Let $a$ be the affine function with $a(t_*)=\omega_0(t_*)$, $a'(t_*)=\omega_0'(t_*)$, and put $\omega:=\omega_0-a$; since $Q_h$ reproduces affine functions, $Q_h\omega-\omega=Q_h\omega_0-\omega_0=D$ --- the defect is unchanged --- while $\omega(t_*)=\omega'(t_*)=0$.

\emph{Step 2: the tower.} Let $\eta>0$ be any number satisfying
\begin{equation}\label{eq:etacond}
\eta\,\|\omega'\|_{L^\infty(J)}\le\frac cL
\qquad\text{and}\qquad
\eta\,L\,\|D\|_{L^\infty(J)}\,e^{c}\le\delta ,
\end{equation}
and put $\varphi:=\mathrm{id}+\eta\omega$.
Then $\Lip(\varphi)\le1+\eta\|\omega'\|_\infty\le1+c/L$, and $\Lip(\widehat\varphi)\le\Lip(\varphi)$ because $Q_h$ is Lip-diminishing; also $\widehat\varphi=Q_h\varphi=\mathrm{id}+\eta Q_h\omega$ by affine reproduction and linearity. Since $\omega(t_*)=0$, the exact trajectory from $x_0=t_*$ satisfies $x_\ell\equiv t_*$.

\emph{Step 3: the recurrence.} With $e_\ell:=\hat x_\ell-t_*$,
\[
e_{\ell+1}=e_\ell+\eta\bigl[Q_h\omega(\hat x_\ell)-Q_h\omega(t_*)\bigr]+\eta D(t_*).
\]
Lip-diminishing applied to $\omega$ bounds $\sup|(Q_h\omega)'|$ on cell interiors by $\|\omega'\|_\infty$, and since $Q_h\omega$ is continuous and piecewise polynomial, hence absolutely continuous, this is also a bound on $\Lip(Q_h\omega)$; so the middle term is bounded in absolute value by $\eta\|\omega'\|_\infty|e_\ell|\le(c/L)|e_\ell|$. Hence
\[
\Bigl(1-\frac cL\Bigr)e_\ell+\eta D(t_*)\ \le\ e_{\ell+1}\ \le\ \Bigl(1+\frac cL\Bigr)e_\ell+\eta\|D\|_\infty .
\]
The lower half and $e_0=0$ give $e_\ell\ge0$ for all $\ell$ and
\[
e_L\ \ge\ \eta D(t_*)\sum_{j<L}\Bigl(1-\frac cL\Bigr)^{j}\ \ge\ \eta D(t_*)\,L\Bigl(1-\frac cL\Bigr)^{L-1}\ \ge\ \eta D(t_*)\,L\,e^{-2c},
\]
using $\log(1-x)\ge-2x$ on $[0,\tfrac12]$ and $L\ge2c$. The upper half gives $e_L\le e^cL\eta\|D\|_\infty\le\delta$ by the second condition in \eqref{eq:etacond}, so $\hat x_\ell=t_*+e_\ell\in J$ for every $\ell$, as required.

\emph{Step 4: conclusion.} $\varepsilon_\ell=\|\varphi-\widehat\varphi\|_\infty=\eta\|Q_h\omega-\omega\|_\infty=\eta\|D\|_\infty$, so $\sum_{\ell<L}\varepsilon_\ell=L\eta\|D\|_\infty$, while
$e_L\ge e^{-2c}L\eta D(t_*)\ge(1-\sigma)e^{-2c}L\eta\|D\|_\infty$.
Finally $\|F-\widehat F\|_{C^0}\ge|F(x_0)-\widehat F(x_0)|=e_L$. In case (a) the quantitative bound on $\|D\|_\infty$ is the one recorded in Step 1; in case (b) no rate in $h$ is claimed.
\end{proof}

\begin{remark}
Three points. (i) The theorem holds for the whole class, with no hypothesis at all beyond membership; but the two branches are not equally quantitative. Non-cancellation itself is universal; the rate $\varepsilon_\ell\gtrsim h^2$ attached to it is available only in branch (a), which contains in particular all operators of derivative order at least two. Branch (b) consists of operators with $m_1\not\equiv0$, which rules out derivative order two or higher --- piecewise linear interpolation is the example --- and for these the size of $\varepsilon_\ell$ is governed by the quadratic defect $E_2$, for which Theorem~\ref{thm:budgetsat} gives the only rate we have. Section~\ref{sec:regimes} states the two levels separately for exactly this reason. (ii) The mechanism is different from Corollary~\ref{cor:sharpbudget}: there the exact trajectory moves and one needs monotonicity of the discretised map, which is a property of Schoenberg's operator and not of the class; here the exact trajectory is a fixed point, and the only structural facts used are $\Lip(Q_h\omega)\le\|\omega'\|_\infty$ and affine reproduction. (iii) The two conditions \eqref{eq:etacond} are explicit and leave $\eta$ free below an explicit threshold; this is what makes Corollary~\ref{cor:prescribed} a statement about a prescribed smoothness ball rather than about a realised value. The budget couples $\eta$ to $L$ through $\eta\|\omega'\|_\infty\le c/L$; consequently $\sum_\ell\varepsilon_\ell$ is $O(1)$ rather than $O(L)$ when the budget is saturated. The content of the theorem is therefore the constant in front of $\sum_\ell\varepsilon_\ell$, not an extra factor of $L$; Section~\ref{sec:regimes} states the consequence with this coupling made explicit.
\end{remark}

\begin{corollary}[Prescribed smoothness ball]\label{cor:prescribed}
Assume in addition $m_1\equiv0$ (branch \textup{(a)}), and replace $Q_h$ by its anchored version $Q_h^{t_c}$, $t_c$ the midpoint of $J$, which belongs to the same class. Assume $h\le|J|/8$. Then $t_*$ may be chosen with $\dist(t_*,\partial J)=|J|/4$ and $|D_0(t_*)|\ge\tfrac18\Theta_k|J|h^2$, and consequently: for every $\kappa_3>0$ with
\begin{equation}\label{eq:kappacond}
\kappa_3\|\omega'\|_{L^\infty(J)}\le\frac cL
\qquad\text{and}\qquad
\kappa_3\,L\,e^{c}\,\|\omega'\|_{L^\infty(J)}\le\frac14
\end{equation}
there is a depth-$L$ tower inside the budget whose edges satisfy $\|\varphi'''\|_\infty=\kappa_3$ exactly and
\[
\|F-\widehat F\|_{C^0}\ \ge\ \tfrac18e^{-2c}\,L\,\kappa_3\,\Theta_k\,|J|\,h^2 .
\]
Both conditions in \eqref{eq:kappacond} are independent of $Q_h$.
\end{corollary}

\begin{proof}
The anchored operator $Q^{t_0}_hf:=Q_hf-(Q_hf-f)(t_0)$ is linear, reproduces affine functions, has unchanged derivative --- hence is Lip-diminishing with the same $\mu_t$, so $m_1,m_2$ are unchanged --- and has range $\Scal_{k,\Delta_h}+\R=\Scal_{k,\Delta_h}$. So it belongs to the class, and its defect is $D_0=D-D(t_0)$.

Take $t_0=t_c$ and $t_\pm=t_c\pm|J|/4$, so $[t_-,t_+]$ has length $|J|/2$. The endpoints need not be grid points, so we may use only the cells entirely contained in $[t_-,t_+]$; their total length is at least $|J|/2-2h\ge|J|/4$ because $h\le|J|/8$, i.e.\ there are at least $|J|/(4h)$ of them. Since $sD'=\tfrac12m_2\ge0$, the function $sD$ is nondecreasing, and summing the per-cell bound $\int_I\tfrac12m_2\ge\Theta_kh^3$ from the proof of Theorem~\ref{thm:T4e} over those cells gives
\[
s\bigl(D(t_+)-D(t_-)\bigr)\ \ge\ \frac{|J|}{4h}\cdot\Theta_kh^3=\tfrac14\Theta_k|J|h^2 .
\]
As $sD(t_-)\le sD(t_c)\le sD(t_+)$ and $D_0(t_c)=0$,
\[
\max\bigl(|D_0(t_+)|,|D_0(t_-)|\bigr)\ \ge\ \tfrac12\bigl|D(t_+)-D(t_-)\bigr|\ \ge\ \tfrac18\Theta_k|J|h^2 ,
\]
so one of $t_\pm$ may serve as $t_*$, with $\dist(t_*,\partial J)=|J|/4$; no constant $(1-\sigma)$ is needed, since $t_*$ is produced explicitly.

It remains to see that \eqref{eq:kappacond} implies \eqref{eq:etacond} with $\delta=|J|/4$ and $\eta=\kappa_3$. Indeed $D_0(t_c)=0$ and
\[
\Lip(D_0)=\Lip(Q_h\omega-\omega)\le\Lip(Q_h\omega)+\Lip(\omega)\le2\|\omega'\|_{L^\infty(J)},
\]
using Lip-diminishing for the first term, so $\|D_0\|_{L^\infty(J)}\le\tfrac{|J|}2\cdot2\|\omega'\|_\infty=|J|\,\|\omega'\|_\infty$; substituting into the second condition of \eqref{eq:etacond} turns it into the second condition of \eqref{eq:kappacond}, which no longer mentions $Q_h$. Finally $\|\varphi'''\|_\infty=\eta\|\omega'''\|_\infty=\eta$ because $\omega$ is a cubic with $\omega'''=\pm1$, and Step 3 of the proof of Theorem~\ref{thm:universal} gives $e_L\ge e^{-2c}L\eta|D_0(t_*)|$.
\end{proof}

\section{Consequences: the cost of exact stability}
\label{sec:regimes}

Target: error $\le\varepsilon$ on $K_0$, depth $L$, budget $1+c/L$, and curvature $\kappa=\sup_\ell\kappa_\ell$ in the sense of Theorem~\ref{thm:Seta}, that is measured by $\|\psi^{(k+1)}\|_\infty$. Regimes~I and II are stated in that normalisation; the lower bounds below use a different one, and we say so where it matters.

\paragraph{Regime I: unconstrained $Q_{G,k}$, $k\ge2$, no projection.}
Two \emph{sufficient} conditions: the budget condition $G\gtrsim\rho(\kappa L/\eta)^{1/k}$ and the accuracy condition $G\gtrsim\rho(\kappa L/\varepsilon)^{1/(k+1)}$. The first dominates for large $L$: the grid is constrained by stability, not by accuracy.

\paragraph{Regime II: the same operator plus rowwise budget projection.}
After projection, rescale row $j$ by $s_{\ell,j}=\min\{1,\lambda_\ell/r_{\ell,j}\}$ with $r_{\ell,j}=\sum_i\widehat M_{\ell,ji}$. The budget then holds identically, $\eta=0$, $\rho=e^cR$. The extra error vanishes when $r_{\ell,j}\le\lambda_\ell$; otherwise
\[
\begin{gathered}
(1-s_{\ell,j})r_{\ell,j}=r_{\ell,j}-\lambda_\ell\le A_\ell,\\
\text{hence}\qquad
\sup_{\|u\|_\infty\le\rho}\bigl\|(I-S_\ell)\widehat\Psi_\ell(u)\bigr\|_\infty\le A_\ell\rho=C_{1,k}\kappa h^k\rho,
\end{gathered}
\]
(the left-hand side is a supremum over the ball, not an operator norm: $(I-S_\ell)\widehat\Psi_\ell$ is not linear), of order $h^k$ rather than $h^{k+1}$, so it dominates the spline error by a factor $\rho/h$, giving $G\gtrsim\rho(e^c\kappa L\rho/\varepsilon)^{1/k}$. Comparing I and II, the exponent in $L$ is the same, $L^{1/k}$. Up to operator-dependent constants --- and noting that $\rho$ differs between the regimes, $e^{c+\eta}R$ against $e^cR$ --- projection becomes preferable in a coarse-accuracy regime, at the scale $\varepsilon\gtrsim\eta\rho$. We deliberately do not state this as an equivalence.

\paragraph{A convention for the lower bounds.}
Throughout the $C^0$ lower-bound discussion, discretisation operators are identified modulo output constants,
\[
Q_h\ \sim\ Q_h+C,\qquad C\in\R\ \text{(possibly depending linearly on }f),
\]
because a KAN discretisation uses the anchored representative of its layers and layer Lipschitz budgets depend only on derivatives. Under this identification the anchored operator $Q^{t_0}_h$ is the same object as $Q_h$, which is what licenses passing to it in Corollary~\ref{cor:prescribed}; it is also why the lower bounds are stated for $\operatorname{osc}$ rather than for the sup-norm.

\paragraph{Regime III: an exactly Lip-diminishing linear scheme.}
Continuous piecewise linear interpolation and Schoenberg's operator $V_G$ of any degree satisfy
$(V_Gf)'=\sum_i\frac{f(\xi_{i+1})-f(\xi_i)}{\xi_{i+1}-\xi_i}N_{i,k-1}$,
a convex combination of divided differences \cite{MarsdenSchoenberg,SchoenbergSplines}, whence $\|(V_Gf)'\|_\infty\le\|f'\|_\infty$, entrywise $\widehat M\le M$, $A_\ell=0$, $\eta=0$ and \emph{no threshold}. The price is the barrier of Section~\ref{sec:rigidity}, now two-sided: derivative order $\le2$ (Theorem~\ref{thm:T4b}) and $C^0$ order $\le2$ (Theorem~\ref{thm:T4e}), both attained (Proposition~\ref{ex:midpoint}, Schoenberg). Piecewise linear interpolation realises $C^0$ order $2$ with derivative order $1$.

Two lower bounds follow, and we keep the smoothness normalisation explicit, writing
\[
\kappa_2:=\sup_\ell\|\psi_\ell''\|_\infty,\qquad \kappa_3:=\sup_\ell\|\psi_\ell'''\|_\infty ,
\]
since the upper bounds for Schoenberg's operator are stated on the ball of $\kappa_2$ and the lower bounds of Theorems~\ref{thm:T4e} and \ref{thm:universal} on the ball of $\kappa_3$: these are different balls and must not both be called $\kappa$.

Three statements must be kept apart; conflating them is the error we most want to avoid.

\emph{(a) Universal non-cancellation, without a rate.} By Theorem~\ref{thm:universal}, for \emph{every} operator of the class there is a depth-$L$ tower inside the budget on which
$\|F-\widehat F\|\ge(1-\sigma)e^{-2c}\sum_\ell\varepsilon_\ell$.
Layer errors therefore need not cancel, whatever the operator. This claim carries no $h$-rate and needs none.

\emph{(b) Class-wide single-layer rate, budget-compatible.} By Theorem~\ref{thm:budgetsat}, on the ball $\{\|\psi'\|_\infty\le B,\ \|\psi'''\|_\infty\le\kappa_3\}$ --- for \emph{arbitrary} $B,\kappa_3>0$ --- every operator of the class has a single-layer error, measured modulo constants, at least
\[
c_*\min\Bigl\{\kappa_{3,\mathrm{eff}}|J|h^2,\ \frac{Bh^3}{|J|^2}\Bigr\},
\qquad
\kappa_{3,\mathrm{eff}}:=\min\Bigl\{\kappa_3,\frac{8B}{|J|^2}\Bigr\} .
\]
In the natural regime $B\ge\kappa_3|J|^2/8$ one has $\kappa_{3,\mathrm{eff}}=\kappa_3$.
Both test functions respect the bound on $\|\psi'\|_\infty$, which is what the layer budget controls; this is the point at which Theorem~\ref{thm:T4e} alone would not suffice, since its degenerate branch escapes by scaling out of the Lipschitz ball. With $h=|J|/G$, a per-layer accuracy $\varepsilon$ forces
\[
G\ \ge\ \min\bigl\{G_1,G_2\bigr\},
\qquad
G_1:=|J|^{3/2}\Bigl(\frac{c_*\kappa_{3,\mathrm{eff}}}{\varepsilon}\Bigr)^{1/2},
\qquad
G_2:=|J|\Bigl(\frac{c_*B}{\varepsilon|J|^{2}}\Bigr)^{1/3}.
\]
The power $|J|^{3/2}$ in the first branch, rather than $|J|$, comes from measuring smoothness by the third derivative; the second branch, with exponent $1/3$, is the price of the operators with $m_1\not\equiv0$.

\emph{(c) The depth-dependent rate, on a prescribed smoothness ball, in the branch $m_1\equiv0$.} By Corollary~\ref{cor:prescribed}, for every operator with $m_1\equiv0$ --- after the harmless replacement of $Q_h$ by its anchored version --- and for every $\kappa_3$ below the two explicit thresholds recorded there, there is a tower inside the budget whose edges satisfy $\|\psi'''\|_\infty=\kappa_3$ \emph{exactly} and
\[
\|F-\widehat F\|\ \ge\ \tfrac18e^{-2c}\,L\,\Theta_k\,\kappa_3\,|J|\,h^2 ,
\qquad\text{hence}\qquad
G\ \ge\ |J|^{3/2}\Bigl(\frac{e^{-2c}\Theta_kL\kappa_3}{8\varepsilon}\Bigr)^{1/2}.
\]
Here $\kappa_3$ is prescribed, not merely realised: the point of Corollary~\ref{cor:prescribed} is that the interior point $t_*$ is produced explicitly, at distance $|J|/4$ from the boundary, so the admissible range of $\kappa_3$ does not depend on where the defect happens to be largest. For operators with $m_1\not\equiv0$ we have no comparable rate, only (b).

\emph{(d) What this does not compare.} The exponent $L^{1/2}$ in (c) is not directly comparable with $L^{1/k}$ of Regime~I: the two refer to different smoothness balls ($\kappa_3$ against $\|\psi^{(k+1)}\|_\infty$), and in Regime~III the budget couples $\kappa_3$ to the depth, since the construction of Theorem~\ref{thm:universal} needs $\kappa_3\|\omega'\|_\infty\le c/L$. If that coupling is saturated the factor $L$ cancels and (c) reduces to $G\gtrsim(c|J|/\varepsilon)^{1/2}$. We therefore state the regimes separately and do not compress them into a single slogan.

\paragraph{What is proved.}
\emph{Exact preservation of the Lipschitz budget by a linear operator with values in a spline space is incompatible with derivative approximation order above two.}

\paragraph{What is not proved.}
In Regime~I it does \emph{not} follow that the threshold $G\gtrsim L^{1/k}$ is necessary for every operator of order $k\ge3$: Theorem~\ref{thm:T4b} forbids a \emph{zero} overshoot but gives no rate, and the overshoot could a priori be much smaller than $h^k$ (Open problem~\ref{op:overshoot}). Thus $G\gtrsim L^{1/k}$ is an explicit sufficient threshold for the chosen high-order $Q_h$, not a class-wide lower bound. The asymmetry is worth stating plainly: in Regime~III the necessity is proved for the whole class, at the single-layer level (Theorem~\ref{thm:T4e}) and across depth (Theorem~\ref{thm:universal}); in Regime~I it is not proved at all.

\section{An expressivity consequence: residual fold suppression}
\label{sec:fold}

\begin{theorem}\label{thm:fold}
Let $I_0,\dots,I_L$ be intervals and let $h_\ell:I_\ell\to I_{\ell+1}$, $h_\ell(t)=t+w_\ell(t)$, with $\Lip(w_\ell;I_\ell)=\theta_\ell<1$; put $H:=h_{L-1}\circ\cdots\circ h_0:I_0\to I_L$. Let $P$ be a piecewise polynomial of degree $k$ on $G$ intervals covering $I_L$, and $F=P\circ H$. Then for every $y$ that is not the value of an identically constant piece,
\[
\#F^{-1}(y)\le kG\qquad\text{uniformly in }L .
\]
\end{theorem}

\begin{proof}
For $y>x$ in $I_\ell$, $h_\ell(y)-h_\ell(x)\ge(1-\theta_\ell)(y-x)>0$, so each $h_\ell$ is strictly increasing, hence injective; the composition is well defined by hypothesis, and $H$ is injective. Therefore $x\mapsto H(x)$ is a bijection from $F^{-1}(y)$ onto $H(I_0)\cap P^{-1}(y)$, whence $\#F^{-1}(y)\le\#P^{-1}(y)$. On each of the $G$ intervals, $P-y$ is a nonzero polynomial of degree $\le k$ and has at most $k$ roots.
\end{proof}

The hypothesis $\Lip(w_\ell)<1$ cannot be relaxed to $\le1$: for $w(t)=-t$ one gets $h_\ell\equiv0$, and injectivity fails at once. Apart from that, only injectivity is used: no surjectivity, no globally defined $w_\ell$, and no invariance $h_\ell(I_\ell)\subseteq I_\ell$. Stating the layers as maps $I_\ell\to I_{\ell+1}$ is what makes the composition well defined, and it makes the theorem more general than the invariant-interval version.

\begin{proposition}[Exponential contrast without the residual constraint]
A tent map $T$ with $m$ monotone branches is a continuous piecewise linear function on $m$ intervals, i.e.\ a legitimate KAN layer of width $1$. Then $T^{\circ L}$ has $m^L$ preimages of a typical $y$, using $O(Lm)$ parameters.
\end{proposition}

We avoid calling this a matching lower bound: it is a contrast between two different constraint classes, not a lower bound within the class of Theorem~\ref{thm:fold}.

\begin{remark}[Why one dimension is safe and the cube in $\R^n$ is not]
Counting points does not grow when restricting to a subinterval, so injectivity suffices. Counting connected components in $n\ge2$ would require control of $P^{-1}(y)\cap H(K_0)$, and $H(K_0)$ may oscillate while keeping a fixed bi-Lipschitz constant. See Open problem~\ref{op:nd}.
\end{remark}

\begin{remark}[Terminology]
Always write \emph{residual} fold suppression. The budget $\lambda_\ell\le1+c/L$ alone does not forbid folds: $x\mapsto|x|$ is $1$-Lipschitz and not injective.
\end{remark}

\section{Further remarks}
\label{sec:remarks}

\paragraph{Oscillation cost of a fixed budget.}
For $g(t)=F(x_0+tv)$ with $\Lip(F)\le e^c$, the total variation on a segment of length $D$ is at most $e^cD$, so the number of consecutive oscillations of amplitude $\ge a$ satisfies $N\lesssim e^cD/a$. A construction with $2^L$ teeth in the sense of \cite{Telgarsky} forces $c=\Omega(L)$, so $e^c$ ceases to be a constant. This is the metric counterpart of Theorem~\ref{thm:fold}; it is a corollary and is not KAN-specific.

\paragraph{Vitushkin.}
The classical obstruction \cite{Vitushkin} concerns the impossibility of a universal exact representation of smooth multivariate functions by fixed superposition schemes of smooth functions of fewer variables. It applies only \emph{jointly} with a bound on $\|\psi^{(k+1)}\|_\infty$ at fixed $n$ and $L$; the Lipschitz budget alone does not produce its hypotheses. To be cited as historical context, not used in any proof.

\paragraph{Generalisation.}
Since $\Lip(\widehat F;K_0)\le e^{c+\eta}$, existing KAN generalisation bounds expressed through layer operator norms or Lipschitz-type complexities \cite{ZhangZhou,LipKAN} apply without a factor exponential in $L$. Caveat: the constant is $L$-free but the complexity of the class is not. For dense layers the number of edge functions between widths $n_\ell$ and $n_{\ell+1}$ is $n_\ell n_{\ell+1}$, so with $O(G)$ coefficients per edge the parameter count is
\[
P_{\mathrm{par}}\asymp G\sum_{\ell=0}^{L-1}n_\ell n_{\ell+1}\ \asymp\ Ln^2G\quad\text{at constant width }n,
\]
and Corollary~\ref{cor:threshold} requires $G\gtrsim L^{1/k}$, hence $P_{\mathrm{par}}\gtrsim n^2L^{1+1/k}$.

\section{Open problems}
\label{sec:open}

\begin{problem}[Refactorisation]\label{op:strong}
For a fixed factorisation the infimum over gauges is computed by Theorem~\ref{thm:balance}: $\Bo(\boldsymbol\Phi)=\opi{M_{L-1}\cdots M_0}^{1/L}$. The remaining functional quantity is
\[
\mathfrak B_L(F)=\inf_{\substack{F=\Phi_{L-1}\circ\cdots\circ\Phi_0\\ \text{fixed widths}}}\opi{{\textstyle\prod_\ell} M_\ell}^{1/L}.
\]
Is there $F$, or a sequence $F_L$, with $\Lip(F)=O(1)$ but $\mathfrak B_L(F)\ge q>1$ uniformly in $L$?
\end{problem}

\begin{problem}[Budget-compatible minimax]\label{op:minimax}
Let $\mathfrak E_h(Q_h;B,K)$ be as in Theorem~\ref{thm:budgetsat}. Is
\[
\inf_{Q_h}\mathfrak E_h(Q_h;B,K)\ \ge\ c\,K_{\mathrm{eff}}|J|h^2
\]
for a natural relation between $B$, $K$ and $|J|$ --- that is, can the second branch of Theorem~\ref{thm:budgetsat}, which loses a factor $h/|J|=1/G$, be removed? Equivalently: must an operator with $m_1\not\equiv0$ pay at least as much, on the Lipschitz ball, as one with $m_1\equiv0$? A positive answer would make the class-wide grid bound of Section~\ref{sec:regimes}(b) uniform in the class, and would remove the case distinction from Section~\ref{sec:regimes}(c).
\end{problem}

\begin{problem}[Nonlinear schemes]\label{op:nonlinear}
The barrier concerns linear operators. Tension-parameter and ENO-type schemes are not covered.
\end{problem}

\begin{problem}[$nD$ folds on a cube]\label{op:nd}
For a global homeomorphism $H:\R^n\to\R^n$ one has $F^{-1}(y)\simeq P^{-1}(y)$, and a bound of Milnor--Thom type \cite{Milnor} gives
\[
b_0\bigl(P^{-1}(y)\bigr)\ \le\ G^n\,C(n,k)
\]
for a constant $C(n,k)$ depending only on $n$ and $k$: exponential in $n$, but not in $L$. We do not fix $C(n,k)$ here. The classical bound applies to an algebraic set, whereas on each cell the object is $\{P=y\}$ intersected with the cell, hence semialgebraic, and a single algebraic component may split when intersected with a cell; the passage also uses $b_0(\bigcup_\alpha X_\alpha)\le\sum_\alpha b_0(X_\alpha)$ for a covering, which holds because every component of the union contains a component of some piece. Making $C(n,k)$ explicit therefore needs a semialgebraic complexity bound rather than Milnor's theorem alone, and is part of the problem rather than a preliminary to it. On a fixed cube $K_0$ the statement does not follow at all.
\end{problem}

\begin{problem}[Quantitative overshoot]\label{op:overshoot}
For $Q_h$ of order $k\ge2$, is
\[
\sup\bigl\{\|(Q_hf)'\|_\infty-\|f'\|_\infty:\ \|f^{(k+1)}\|_\infty\le1,\ \|f'\|_\infty\le1\bigr\}\ \ge\ c\,h^k\ ?
\]
A positive answer would turn the sufficient threshold $G\gtrsim L^{1/k}$ into a necessary one for the whole class.

\emph{Strategy, not a proof.} For a shift-invariant quasi-interpolant reproducing $\Pcal_k$ on a uniform grid and $f_\star(x)=x^{k+1}/(k+1)!$, one has $Q_hf_\star-f_\star=h^{k+1}\phi(x/h)$ with a fixed nonzero periodic $\phi$, hence $(Q_hf_\star)'-f_\star'=h^k\phi'(x/h)$. One then needs $g$ whose $g'$ has a flat maximum of width $\sim h$ located where $\phi'>0$, e.g.\ $g'=1-A(x-x_0)^{2m}$, optimising $A$ and $m$ against $\|g^{(k+1)}\|_\infty$. The obstruction to a cheap argument: if the maximum of $g'$ is attained where $g$ is affine, locality together with affine reproduction gives zero overshoot.
\end{problem}

\appendix

\section{Degenerate cases of Theorem~\ref{thm:balance}}
\label{sec:degenerate}

\subsection{Zero components of \texorpdfstring{$u_\ell$}{u-l}, with \texorpdfstring{$\tau_*>0$}{tau>0}}

Let $Z_\ell=\{i:u_{\ell,i}=0\}$ and $R_\ell=[n_\ell]\setminus Z_\ell$; since $u_0=\one>0$, $Z_0=\varnothing$.

\begin{lemma}[Structure]\label{lem:dead}
If $i\in Z_{\ell+1}$ then row $i$ of $M_\ell$ is supported on $Z_\ell$.
\end{lemma}

\begin{proof}
$0=u_{\ell+1,i}=\sum_j(M_\ell)_{ij}u_{\ell,j}$ with all summands $\ge0$, so $(M_\ell)_{ij}>0$ implies $u_{\ell,j}=0$.
\end{proof}

Thus $Z_\ell$ consists of hidden coordinates unreachable from the input along any path of nonzero weight (\emph{derivative-dead}), and only such coordinates feed into them. For $\delta\in(0,1)$ set
\[
d^\delta_{\ell,i}=
\begin{cases}
\tau_*^{-\ell}u_{\ell,i},& i\in R_\ell,\\
\delta^{\,L-\ell},& i\in Z_\ell,
\end{cases}
\qquad \ell=1,\dots,L-1 ,
\]
and write $s_j$ for the $j$-th row sum of $D_{\ell+1}^{-1}M_\ell D_\ell$.

\begin{itemize}[nosep]
\item $\ell\le L-2$, $j\in R_{\ell+1}$: since $u_{\ell,i}=0$ on $Z_\ell$, the sum over $R_\ell$ already gives the full $\tau_*^{-\ell}u_{\ell+1,j}$, so
$s_j=\tau_*+\dfrac{\tau_*^{\ell+1}}{u_{\ell+1,j}}\delta^{L-\ell}\sum_{i\in Z_\ell}(M_\ell)_{ji}\to\tau_*$.
\item $\ell\le L-2$, $j\in Z_{\ell+1}$: by Lemma~\ref{lem:dead} the row sits on $Z_\ell$, so $s_j=\delta\sum_{i\in Z_\ell}(M_\ell)_{ji}\to0$, since $\delta^{-(L-\ell-1)}\delta^{L-\ell}=\delta$.
\item $\ell=L-1$, $D_L=I$:
$s_j=\tau_*^{-(L-1)}u_{L,j}+\delta\sum_{i\in Z_{L-1}}(M_{L-1})_{ji}\le\tau_*+\delta C\to\tau_*$.
\end{itemize}

Hence $\limsup_{\delta\to0}\max_\ell\opi{\cdot}\le\tau_*$, and with the lower bound $\Bo=\tau_*$. Equivalently: delete derivative-dead coordinates first, then apply the nondegenerate case.

\subsection{The case \texorpdfstring{$P=0$}{P=0}}

Then $\tau_*=0$ and we must show $\Bo=0$. Call $i$ at layer $\ell$ \emph{live-in} if $u_{\ell,i}>0$ and \emph{live-out} if $(\one^\top M_{L-1}\cdots M_\ell)_i>0$. The hypothesis $P=0$ means there is no positive-weight path from input to output, i.e.\ no coordinate is simultaneously live-in and live-out. Put
\[
d^\delta_{\ell,i}=
\begin{cases}
\delta^{-\ell},& i\ \text{live-in},\\
\delta^{\,L-\ell},&\text{otherwise.}
\end{cases}
\]
The verification is a single observation about edges rather than three cases. Along any edge $i\to j$ of nonzero weight, if $i$ is live-in then so is $j$; contrapositively, if $j$ is not live-in then neither is $i$. Hence for a nonzero entry $(M_\ell)_{ji}$ the scaling factor $d^\delta_{\ell,i}/d^\delta_{\ell+1,j}$ equals $\delta^{-\ell}/\delta^{-(\ell+1)}=\delta$ when both endpoints are live-in, and $\delta^{L-\ell}/\delta^{L-\ell-1}=\delta$ when neither is; the mixed case ``$i$ live-in, $j$ not'' is empty by the observation, and the case ``$j$ live-in, $i$ not'' contributes $\delta^{L-\ell}/\delta^{-(\ell+1)}\le\delta$. For the last layer, $D_L=I$ and any node feeding a nonzero output row is live-out, hence not live-in under $P=0$, so it is again weighted by $\delta$. Every nonzero edge therefore carries a factor $\delta$, and $\max_\ell\opi{\cdot}=O(\delta)\to0$.

\subsection{Non-attainment}

\begin{example}\label{ex:nonattain}
$L=2$, $n_0=n_2=1$, $n_1=2$, $M_0=\binom{1}{0}$, $M_1=(1,\ a)$ with $a>0$. Then $P=1$, $\tau_*=1$, $Z_1=\{2\}$, and for every $D_1=\diag(d_1,d_2)>0$,
\[
\max\Bigl(\tfrac1{d_1},\ d_1+ad_2\Bigr)>\max\Bigl(\tfrac1{d_1},\,d_1\Bigr)\ge1=\tau_* .
\]
The infimum equals $\tau_*$ but is not attained. This is exactly a dead neuron: all incoming edges have zero Lipschitz constant while outgoing ones do not.
\end{example}

\section{Explicit constants}
\label{sec:constants}

We collect the constants that enter the statements, with the values we have been able to compute. Throughout, $d=\max(k-1,2)$ in Lemma~\ref{lem:normeq} and $d=\max(k-1,1)$ in Definition~\ref{def:kink}.

\paragraph{The second moment.} By Lemma~\ref{lem:secondmoment}, for $k\ge2$ and a uniform grid,
\[
M_2\equiv\frac{k+1}{12}h^2 ,
\]
exactly and independently of $t$. This is the only constant of Section~\ref{sec:sharp}, and it is the reason the sharpness constant there is exactly $e^{-2c}$.

\paragraph{The reference kink.} By Definition~\ref{def:kink}, $c_k=\dist_{L^\infty[-1/2,1/2]}(G_*,\Pcal_{\max(k-1,1)})$ with $G_*(x)=\tfrac12x^2-x_+^2=-\tfrac12x|x|$. Since $G_*$ is odd, best approximation by $\Pcal_{2m}$ coincides with best approximation by $\Pcal_{2m-1}$, so $c_k$ takes the same value for consecutive $k$. For $\max(k-1,1)\in\{1,2\}$, i.e.\ $k\le3$, the extremal polynomial is $-\tfrac12(2\sqrt2-2)x$ and
\[
c_k=\frac{3-2\sqrt2}{8}=0.0214466\ldots
\]
Numerically, $c_k=3.45432\cdot10^{-3}$ for $k\in\{4,5\}$ and $c_k=1.32942\cdot10^{-3}$ for $k=6$. In particular $c_k$ is nonincreasing in $k$ --- immediately, since the spaces $\Pcal_{\max(k-1,1)}$ are nested --- so the barriers of Section~\ref{sec:rigidity} do not strengthen as the degree grows; no asymptotic rate in $k$ is claimed.

\paragraph{The norm equivalence.} In Lemma~\ref{lem:normeq}, $\gamma_d$ is the reciprocal of the Nikolskii constant of $\Pcal_d$ for the pair $(L^\infty,L^1)$ on an interval. Testing on $q\equiv1$ gives $\gamma_d\le1$, which is what makes the smallness condition in Theorem~\ref{thm:budgetsat} automatic. For $d=1$, $\gamma_1=\sqrt2-1=0.414214\ldots$; numerically $\gamma_2=0.2278$, $\gamma_3=0.1443$, $\gamma_4=0.0997$, $\gamma_5=0.0730$, $\gamma_6=0.0558$, to the digits shown.

\paragraph{The signed integral.} The proof of Lemma~\ref{lem:signed} gives $\beta_d\ge\bigl(2(d+1)^2\bigr)^{-1}$ through Markov's inequality; we do not claim this is sharp.

\paragraph{Composite constants.} Consequently
\[
\Theta_k=\gamma_dc_k,
\qquad
c_*=\min\bigl\{\tfrac14\gamma_dc_k,\ \tfrac12\beta_d\gamma_dc_k\bigr\}=\tfrac12\beta_d\gamma_dc_k
\]
whenever $\beta_d\le\tfrac12$, which holds for every $d\ge1$ by the previous paragraph. For $k=3$, for instance, $d=2$ in Definition~\ref{def:kink} and $\Theta_3=0.227774\cdot0.0214466=4.885\cdot10^{-3}$.

\paragraph{The quasi-interpolant.} Fix the reference nodes of Lemma~\ref{lem:functionals} to be the Chebyshev--Lobatto points of $[0,1]$,
\[
s_r=\tfrac12\Bigl(1-\cos\tfrac{r\pi}{k}\Bigr),\qquad r=0,\dots,k
\]
(for $k=1$, $s_0=0$, $s_1=1$). With $R_k=(2k+1)/2$ the chain of estimates in the proof of Lemma~\ref{lem:existence} gives
\[
C_{0,k}=(1+\Lambda_k\kappa_k)\frac{R_k^{k+1}}{(k+1)!},
\qquad
C_{1,k}=2k^2\Lambda_k\kappa_k\frac{R_k^{k+1}}{(k+1)!}+\frac{R_k^{k}}{k!} .
\]
Both $\Lambda_k$ and $\kappa_k$ are computable once and for all on the reference configuration. The values are

\begin{center}
\begin{tabular}{cccrr}
\hline
$k$ & $\Lambda_k$ & $\kappa_k$ & $C_{0,k}$ & $C_{1,k}$\\
\hline
$1$ & $1.0000$ & $1.0000$ & $2.25$ & $3.75$\\
$2$ & $1.2500$ & $1.2222$ & $6.58$ & $34.95$\\
$3$ & $1.6667$ & $1.1667$ & $18.41$ & $225.99$\\
$4$ & $1.7988$ & $1.6485$ & $60.98$ & $1476.25$\\
$5$ & $1.9889$ & $1.5235$ & $154.93$ & $5866.28$\\
\hline
\end{tabular}
\end{center}

\noindent
For $k=1$ the two constants are exact: the nodes are the endpoints, so $\Lambda_1=1$, and the B-spline coefficients of a linear polynomial are its values at the Greville abscissae, which are the knots, so $\kappa_1=1$; hence $C_{0,1}=2\cdot(3/2)^2/2=9/4$ and $C_{1,1}=2\cdot9/8+3/2=15/4$. No attempt has been made to optimise these constants: replacing the Taylor polynomial by a near-best polynomial approximation on $U_i$ would gain a factor of order $2^{k}$ in $C_{0,k}$, at the cost of a longer argument.

\paragraph{What the theorems actually need.} None of the statements of this paper depends on the numerical values above. Theorem~\ref{thm:T4d} needs only $c_k>0$, which is Definition~\ref{def:kink}; Lemma~\ref{lem:normeq} needs only $\gamma_d>0$, and the one place where a numerical fact is used, namely $\gamma_d\le1$ in Theorem~\ref{thm:budgetsat}, is proved by testing on $q\equiv1$; Lemma~\ref{lem:signed} carries its own explicit $\beta_d$; and Lemma~\ref{lem:secondmoment} is an identity. The table is reported because explicit constants are more useful than existential ones, not because anything rests on them.

\paragraph{How the constants were computed.} Each of $c_k$, $\gamma_d$ and $\kappa_k$ is the value of an extremal problem over a finite-dimensional space, and each becomes a linear program after discretising the interval by a uniform grid $\{x_1,\dots,x_N\}$:
\[
\begin{aligned}
c_k&:\quad \min_{p\in\Pcal_{\max(k-1,1)},\,E\ge0}E
&&\text{subject to } |G_*(x_m)-p(x_m)|\le E,\\
\gamma_d^{-1}&:\quad \max_{j}\ \max_{q\in\Pcal_d}\ |q(x_j)|
&&\text{subject to } \textstyle\sum_m w_m|q(x_m)|\le1,\\
\kappa_k&:\quad \max_{q\in\Pcal_k}\ c_i(q)
&&\text{subject to } |q(x_m)|\le1 ,
\end{aligned}
\]
with $w_m$ the trapezoidal weights and $c_i$ the coefficient functional of Lemma~\ref{lem:functionals}, represented on the monomial basis. For $\gamma_d$ the ratio $\|q\|_1/\|q\|_\infty$ is scale invariant, so it is minimised directly over the coefficient space rather than through the displayed linear program. The Lebesgue constant $\Lambda_k$ is the maximum of the Lebesgue function, evaluated directly.

All of these are reported as \emph{numerical approximations, not certified enclosures}. The directions differ: sampling the interval at finitely many points makes the computed minimax value for $c_k$ a lower bound, whereas replacing $\|q\|_{L^\infty(J_i)}\le1$ by finitely many constraints $|q(x_m)|\le1$ enlarges the feasible set and makes the computed $\kappa_k$ an upper bound; for $\gamma_d$, where a quadrature rule and a sampled maximum enter simultaneously, and for the sampled Lebesgue maximum, we claim no one-sided certification. Consequently the table of $C_{0,k}$ and $C_{1,k}$ should be read as numerical estimates. This is harmless: as recorded above, the proofs use only the positivity of the constants.

The values were numerically stable under substantial grid refinement to the precision displayed; the accompanying code has a \texttt{-{}-refinement-check} mode which prints each grid-dependent quantity on $N,2N,4N,8N$, so this statement is reproducible from the public script. The exact benchmark checks used for the implementation are $c_k=(3-2\sqrt2)/8$ for $k\le3$, $\gamma_1=\sqrt2-1$, $\Lambda_1=\kappa_1=1$, and the identity of Lemma~\ref{lem:secondmoment}, which is reproduced to twelve digits and which was found numerically and verified symbolically for $k=2,3,4$ before being proved.

\section*{Data availability}
No empirical data were used. Code reproducing the auxiliary numerical constants reported in Appendix~\ref{sec:constants}, namely the Lebesgue constants, the norms of the coefficient functionals, the best-approximation constants $c_k$ and the Nikolskii constants $\gamma_d$, is publicly available at \url{https://github.com/hippocratus/math/tree/v1-arxiv/spline-kan-constants}.

\end{document}